\documentclass{article}

\usepackage{microtype}
\usepackage{graphicx}
\usepackage{booktabs} 

\usepackage{hyperref}

\usepackage[accepted]{icml2024}

\usepackage{amsmath,amssymb,amsthm}
\usepackage{bbm}
\usepackage{xcolor}
\usepackage{mathtools}
\usepackage{subcaption}
\usepackage{enumerate}

\theoremstyle{plain}
\newtheorem{theorem}{Theorem}[section]
\newtheorem{proposition}[theorem]{Proposition}
\newtheorem{lemma}[theorem]{Lemma}

\theoremstyle{definition}
\newtheorem{definition}[theorem]{Definition}
\newtheorem{assumption}[theorem]{Assumption}
\theoremstyle{remark}

\DeclareMathOperator*{\argmin}{\operatorname{\arg\min}}
\DeclareMathOperator*{\argmax}{\operatorname{\arg\max}}

\usepackage[textsize=tiny]{todonotes}

\icmltitlerunning{Stochastic Bandits with ReLU Neural Networks}

\begin{document}

\twocolumn[
\icmltitle{Stochastic Bandits with ReLU Neural Networks}




\begin{icmlauthorlist}
\icmlauthor{Kan Xu}{asu}
\icmlauthor{Hamsa Bastani}{penn}
\icmlauthor{Surbhi Goel}{penn}
\icmlauthor{Osbert Bastani}{penn}
\end{icmlauthorlist}

\icmlaffiliation{asu}{Arizona State University, Arizona, USA}
\icmlaffiliation{penn}{University of Pennsylvania, Pennsylvania, USA}

\icmlcorrespondingauthor{Kan Xu}{kanxu1@asu.edu}

\icmlkeywords{Machine Learning, ICML}

\vskip 0.3in
]



\printAffiliationsAndNotice{}  

\begin{abstract}
We study the stochastic bandit problem with ReLU neural network structure. We show that a $\tilde{O}(\sqrt{T})$ regret guarantee is achievable by considering bandits with one-layer ReLU neural networks; to the best of our knowledge, our work is the first to achieve such a guarantee. In this specific setting, we propose an OFU-ReLU algorithm that can achieve this upper bound. The algorithm first explores randomly until it reaches a \emph{linear} regime, and then implements a UCB-type linear bandit algorithm to balance exploration and exploitation. Our key insight is that we can exploit the piecewise linear structure of ReLU activations and convert the problem into a linear bandit in a transformed feature space, once we learn the parameters of ReLU relatively accurately during the exploration stage. To remove dependence on model parameters, we design an OFU-ReLU+ algorithm based on a batching strategy, which can provide the same theoretical guarantee.\footnote{Source code is available at \url{https://github.com/kanxu526/ReLUBandit}.}
\end{abstract}

\section{Introduction}

The stochastic contextual bandit problem has been widely studied in the literature \citep{bubeck2012regret,lattimore2020bandit}, with broad applications in healthcare \citep{bastani2020online}, personalized recommendation \citep{li2010contextual}, etc. The problem is important since real-world decision-makers oftentimes adaptively gather information about their environment to learn. Formally, the bandit algorithm actively selects a sequence of actions $\{x_t\}_{t\in[T]}$ with $x_t\in\mathcal{X}$ over some horizon $T\in\mathbb{N}$, and observes stochastic rewards $y_t=f_{\Theta^*}(x_t)+\xi_t$, where $f_{\Theta^*}$ is the true reward function (represented by a model with parameters $\Theta^*$), and $\xi_t$ is random noise. Thus, to achieve good performance or low regret, the decision-maker must maintain small decision error uniformly across all actions $x_t$ over time to ensure that it generalizes to new, actively selected actions.

With the success of bandits in practice, there has been a great deal of recent interest in understanding the theoretical properties of bandit algorithms. For linear models, i.e., the expected reward $f_{\Theta^*}$ is linear in $x_t$, bandit algorithms have adapted techniques from statistics to address this challenge \citep{dani2008stochastic,rusmevichientong2010linearly,abbasi2011improved}. Particularly, linear bandit algorithms build on linear regression, which provides parameter estimation bounds of the form $\|\hat\theta-\theta^*\|_2\le\epsilon$; then, we obtain uniform generalization bounds of the form
$|f_{\hat\theta}(x)-f_{\theta^*}(x)|\le L\epsilon, \forall x\in\mathcal{X}$,
where $L$ is a Lipschitz constant for $f_\theta$. By adapting these techniques, linear bandit algorithms can achieve minimax rates of $\tilde{O}(\sqrt{T})$~\citep{dani2008stochastic,abbasi2011improved} in terms of regret.
However, the linear assumption oftentimes do not hold for complicated tasks \citep{valko2013finite}, and has recently motivated the study of nonlinear contextual bandits, especially building upon neural networks with ReLU activations \citep[see, e.g.,][]{zhou2020neural,zhang2020neural,xu2020neural,kassraie2022neural}. 

One of the key questions is what kinds of guarantees can be provided in such settings when the underlying true reward model has ReLU structures. The current existing literature of bandits based on ReLU neural networks mainly base their analyses on the theory of neural tangent kernel (NTK)~\citep{jacot2018neural}. \citet{zhou2020neural,gu2024batched} leverage NTK and upper confidence bound (UCB) techniques to achieve a $\tilde{\mathcal{O}}(\gamma_T \sqrt{T})$ regret bound, where $\gamma_T$ is the effective dimension or maximum information gain and is assumed to be $T$-independent in these literature. This assumption is strong because it intuitively assumes the function is linear on a low-dimensional subspace (more precisely, it says that the eigenvalues of the empirical covariance matrix in the kernel space vanish quickly, so the covariates ``approximately'' lie in a low-dimensional subspace); thus, it effectively converts the problem to linear bandit problem in a high-dimensional subspace. Indeed, for ReLU neural networks on a $d$-dimensional domain (i.e., $x\in\mathbb{R}^d$), \cite{kassraie2022neural} shows a best upper bound for the information gain known as $\gamma_T=\tilde{O}(T^{\frac{d-1}{d}})$, even for a one-layer neural network. Consequently, the regret bound provided above becomes superlinear even for $d>1$ without further restrictive assumptions.  \cite{kassraie2022neural} improve upon this regret bound based on a variant of \cite{zhou2020neural} and obtain a sublinear bound of $\tilde{O}((\gamma_TT)^{1/2})=\tilde{O}(T^{\frac{2d-1}{2d}})$, but is still far from the typical $\tilde{O}(\sqrt{T})$ guarantee.

\textbf{Contribution.} In contrast, we want to shed light on the achievable regret guarantee for a nonlinear bandit problem with ReLU neural network structure by estimating the model $f_{\Theta^*}$ directly (without making an effective dimension assumption using NTK techniques). Due to the complexity of the problem, we consider the setting of one-layer ReLU neural network as the true reward function. We design two bandit algorithms that exploit the piecewise linear structure of ReLU activations; \emph{to the best of our knowledge, we provide the first $\tilde{\mathcal{O}}(\sqrt{T})$ regret bound for bandit learning with ReLU neural networks.}

Our first bandit algorithm OFU-ReLU is designed based on the following insight. Let our true reward function have the following ReLU structure with $k$ neurons \citep[see, e.g.,][]{du2018power,zhang2019learning}
\begin{align*}
f_{\Theta^*}(x) = \sum_{i\in[k]} \theta_i^{*\top} x\cdot\mathbbm{1}(\theta_i^{*\top} x\ge0),
\end{align*}
where $\Theta^*=[\theta_1^{*},\cdots,\theta_k^*]^\top$.
Intuitively, once we learn a sufficiently good estimate $\tilde\theta_i$ that is close to its corresponding true neuron parameter $\theta^*_i$, we can ``freeze'' the contribution of the indicator function by $\mathbbm{1}(\tilde\theta_i^{\top} x\ge0)$. The problem then becomes linear, and we can use a UCB-type linear bandit algorithm to achieve good performance. 
This insight motivates our two-phase bandit algorithm, where we first randomly explore to estimate the parameters, and then run linear bandits once we reach the linear regime to obtain at a minimax optimal rate. 

We want to note that even though this strategy may seem straightforward, such a two-stage design with a phase transition from exploration to bandit learning has been shown inevitable for specific nonlinear bandit problems \citep{rajaraman2023beyond}, and might be a more general phenomenon. In addition, applying UCB-type algorithm (i.e., Eluder UCB \citep{russo2013eluder}) directly on the reward has been demonstrated suboptimal for certain family of nonlinear bandit problem, posing an interesting theoretical challenge \citep{rajaraman2023beyond}. \emph{In contrast, we show that instead of applying UCB to the ReLU structure directly, we will be able to reduce it to a linear bandit problem and this will make UCB optimal again.}

Our second bandit algorithm OFU-ReLU+ is designed to eliminate the assumed knowledge in OFU-ReLU of the minimum gap $\nu_*$ between the optimal action $x^*$ and any neuron $\theta^*_i$ (we prove such a gap always exists for our ReLU structure). 
As long as the estimate $\tilde\theta_i$ is within $\nu_*/2$ of the true neuron $\theta_i^*$, the indicator estimate is guaranteed to be consistent with the true indicator value at the optimal action $x=x^*$, i.e., $\mathbbm{1}(\tilde\theta_i^{\top} x^*\ge0)=\mathbbm{1}(\theta_i^{*\top} x^*\ge0)$. In other words, unless our estimate $\tilde\theta$ has less than a $\nu_*$-dependent error, the optimal action $x^*$ would lie near a nonlinear regime of a neuron, which might fail the following linear bandit learning. Yet, in practice we do not know $\nu_*$; thus, we design a batching strategy that first makes a guess on $\nu_*$ and keep cutting this guess every batch so that our estimate $\tilde\theta_i$ will be accurate enough after a constant number of batches.
\emph{Different from the previous batching strategies \citep[see, e.g.,][]{golrezaei2019dynamic,luo2022contextual}, our exploration and OFUL phase both use samples from all previous batches without discarding data from previous batches. }

\emph{Finally, we provide a parameter estimation error bound for one-layer ReLU neural networks that can ensure theoretical guarantee for each neuron independently through a novel proof strategy, which might be of separate interest.}

\subsection{Other Related Work}

Early work on stochastic bandits focus on linear rewards~\citep{dani2008stochastic,chu2011contextual,abbasi2011improved} and typically use an UCB algorithm. Later, it has been extended to kernel based models, where the reward functions belong to the reproducing kernel Hilbert space (RKHS) \citep{valko2013finite}. Along this line, the recent ReLU bandit literature are build upon the kernelized algorithm using NTK and achieve an effective-dimension-dependent bound~\citep{zhou2020neural,xu2020neural,kassraie2022neural,gu2024batched}; \cite{salgia2023provably} generalize to smooth activations and provide a bound $\tilde{O}(T^{\frac{2d+2s-3}{2d+4s-4}})$ depending on smoothness $s$. 
\citet{dong2021provable} study the optimization scheme of nonlinear bandit, and provide a $\tilde{\mathcal{O}}(T^{3/4})$ and $\tilde{\mathcal{O}}(T^{7/8})$ local and global regret for two-layer neural network bandit.
\cite{rajaraman2023beyond} focus specifically on the ridge function family (which does not include ReLU but only ReLU with single neuron), and design an explore-then-commit strategy. Finally, there are also many other works that study neural network bandits with other activations, such as quadratic activations \citep{xu2021robust}, polynomial functional form \citep{huang2021optimal}, etc.

\section{Problem Formulation}

\textbf{Notation.} Let $[n]=\{1, \cdots,n\}$. We let $S^{d-1}(r) \subseteq \mathbb{R}^{d}$ denote the $(d-1)$-sphere in $d$ dimensions with radius $r$, and let $S^{d-1}=S^{d-1}(1)$. Define $A^{d-1}(r)$ to be the area of the $(d-1)$-sphere with radius $r$ (i.e., $A^{d-1}(r)=|S^{d-1}(r)|$). We use $\lor$ to represent the maximum value. Define $p$ to be the distribution of the covariates, i.e., $x\sim p$.

\textbf{ReLU Neural Network.} Consider a function family $f_{\Theta}:\mathcal{X}\to\mathcal{Y}$ (where $\mathcal{X}=S^{d-1}\subseteq\mathbb{R}^d$, $\mathcal{Y}=\mathbb{R}$, and $\Theta\in\varTheta$ for a given domain $\varTheta\subset \mathbb{R}^{k\times d}$) consisting of neural networks with ReLU activations \citep[see, e.g.,][]{du2018power,zhang2019learning}
\begin{align}\label{def:relu}
f_{\Theta}(x)=\sum_{i\in[k]}g(\theta_i^\top x),
\end{align}
where $\theta_i$ is the $i^{th}$ component of the parameter $\Theta=[\theta_1, \theta_2, \cdots, \theta_k]^\top$ (which we call a \emph{neuron}) and $g(z)=z\cdot\mathbbm{1}(z\ge0)$. Let $\Theta^*$ be the ground-truth parameters. 

We are provided with $n$ data points $Z=\{(x_i,y_i)\}_{i\in[n]}$ such that $y_i=f_{\Theta^*}(x_i)+\xi_i$, where $\xi_i$ is i.i.d. $\sigma$-subgaussian random noise with mean $0$. 
Assume the covariates follow certain distribution $x\sim p$. Then, the population mean squared loss is defined as
\begin{align*}
L_{S,p}(\Theta)=\mathbb{E}_{p}\left[(f_\Theta(x)-f_{\Theta^*}(x))^2\right].
\end{align*}
We slightly abuse our notation and use $p$ to denote the joint distribution of each training example $(x_i, y_i)$; the corresponding empirical mean squared loss is
\begin{align}\label{def:empLs}
\hat{L}_S(\Theta;Z)=\frac{1}{n}\sum_{i\in[n]}(f_{\Theta}(x_i)-y_i)^2.
\end{align}
For some neural network $f_{\Theta^*}$, we obtain an estimate $\hat\Theta$ for the parameter $\Theta^*$ by taking the minimizer of $\hat{L}_S(\Theta;Z)$, i.e., $\hat\Theta=\argmin_{\Theta\in\varTheta} \hat{L}_S(\Theta;Z)$.

\textbf{Assumptions.} We provide a statistical guarantee for parameter estimation of ReLU neural network under the following assumptions.
\begin{assumption}
\label{assump:neuronnorm}
$\|\theta_i^*\|_2=1$ holds for all neurons $i\in[k]$.
\end{assumption}
Note that under the above assumption, the domain of $\Theta$, i.e., $\varTheta$, is included in $\{\Theta\in\mathbb{R}^{k\times d} \mid \|\Theta\|_{2,1}=k\}$, where $\|\Theta\|_{2,1}\coloneqq\sum_{i\in[k]}\|\theta_i\|_2$ is the $\ell_{2,1}$-norm.
\begin{assumption}
\label{assump:neuronskip}
There exists a constant $\alpha_0>0$ such that
\begin{align*}
\min_{j,j'\in[k], j\ne j'}\|\theta^*_j\pm\theta^*_{j'}\|_2 \ge \alpha_0.
\end{align*}
\end{assumption}
Collectively, our assumptions limit the structure of the neural network. As evidenced by lower bounds in~\cite{dong2021provable}, some restrictions on the structure are necessary to obtain regret $\tilde{\mathcal{O}}(T^\alpha)$ for $\alpha<1$. 

Assumption~\ref{assump:neuronnorm} is critical for our analysis, since neurons become hard to estimate when they are small. Similarly, Assumption~\ref{assump:neuronskip} is necessary since two neurons that are close together are hard to distinguish. The assumption that the second layer consists of weights ones is not critical---most of our proofs can be extended to the case where the second layer has values in $\{\pm1\}$.

\textbf{ReLU Bandit.} We consider the following bandit problem. At each step $t$, we choose an action $x_t\in\mathcal{X}$, and observe reward
\begin{align}\label{def:rwd_bandit}
y_t=f_{\Theta^*}(x_t)+\xi_t,
\end{align}
where $\xi_t$ is i.i.d. $\sigma$-subgaussian random noise with mean $0$. Here, the true parameters $\Theta^*$ are unknown and will be learned in an online manner. Our goal is to minimize the cumulative regret $R_T$ over a time horizon $T$:
\begin{align*}
R_T = \sum_{t=1}^Tr_t, \quad r_t=f_{\Theta^*}(x^*)-f_{\Theta^*}(x_t)
\end{align*}
where $x^*=\operatorname*{\argmax}_{x\in\mathcal{X}}f_{\Theta^*}(x)$ is the optimal action and $r_t$ is the per period regret.

\section{Parameter Estimation for ReLU Neural Networks}
\label{sec:relunetwork}

We provide an estimation error bound on the parameters of ReLU neural networks, which is important for the convergence of our bandit algorithm. Since the population loss function is symmetric regarding the neurons, meaning that any column-permutation of $\Theta^*$ achieves zero loss, our bound shows the existence of some mapping $\sigma:[k]\to[k]$ from the ground truth neurons $\theta^*_{i}$ to the estimated neurons $\hat\theta_{\sigma(i)}$ such that $\hat\theta_{\sigma(i)}\approx\theta^*_{i}$. 

First, we show the following proposition, which states that a small generalization error on a special subset of the unit sphere $X_i(\epsilon)\subset\mathcal{X}=S^{d-1}$ for $\epsilon\in\mathbb{R}_{>0}$, i.e.,
\begin{align*}
X_i(\epsilon)=\{x\in\mathcal{X}\mid|\theta_i^{*\top} x|\le\epsilon\},
\end{align*}
implies a small parameter estimation error of the corresponding neuron $\theta_i^*$ up to a sign flip.
\begin{proposition}
\label{prop:key}
Suppose 
\begin{align*}
L_{X_i(\epsilon)}(\Theta)\coloneqq\int_{X_i(\epsilon)}|f_\Theta(x)-f_{\Theta^*}(x)|dx
&\le\eta
\end{align*}
for some $\eta\in\mathbb{R}_{>0}$.
Then, there exists a bijection $\sigma:[k]\to[k]$ such that 
\begin{align*}
\min\{\|\theta_{\sigma(i)}-\theta_i^*\|_2,
\|\theta_{\sigma(i)}+\theta_i^*\|_2\}\le h(\eta,\epsilon),
\end{align*}
where $h(\eta,\epsilon)\coloneqq\frac{k\epsilon^3|S^{d-3}|/2}{\epsilon^2(1-d\epsilon^2/2)|S^{d-2}|/8-\eta-6kd\epsilon^3|S^{d-2}|}$.
\end{proposition}
We provide a detailed proof with illustrations in Appendix \ref{app:pf_prop}.
Intuitively, $X_i$ contains those $x\in\mathcal{X}$ close to the boundary where the neuron $\theta_i^*$ of $f_{\Theta^*}$ is nonlinear (i.e., $\theta_i^{*\top}x=0$). Then, this proposition claims that if the loss on certain $X_i$ is small, then the corresponding neuron $\theta_i^*$ in the ground-truth neural network $f_{\Theta^*}$ can be identified up to a sign flip by a corresponding neuron $\theta_{\sigma(i)}$ in the approximate neural network $f_{\Theta}$, i.e., $\theta_{\sigma(i)}\approx\theta_i^*$ or $\theta_{\sigma(i)}\approx-\theta_i^*$. Moreover, if the loss on all $X_i$'s with $i\in[k]$ are small, then $\Theta$ is close to $\Theta^*$ via the mapping $\sigma$ up to sign flips.

Our main result combines Proposition~\ref{prop:key} with a standard generalization error bound for ReLU neural networks to obtain the following parameter estimation error bounds.
\begin{theorem}
\label{thm:key}
Suppose the distribution $p$ satisfies
\begin{align*}
\frac{1}{|S^{d-1}|}\int_{\mathcal{X}}|f_\Theta(x)-f_{\Theta^*}(x)|dx
\le\mathbb{E}_{p}\left[|f_\Theta(x)-f_{\Theta^*}(x)|\right].
\end{align*}
Then, there exists a bijection $\sigma:[k]\to[k]$ such that
\begin{align*}
\min\{\|\hat\theta_{\sigma(i)}-\theta_i^*\|_2, \|\hat\theta_{\sigma(i)}+\theta_i^*\|_2\}\le727\pi^{-\frac{1}{4}}kd^{\frac{1}{4}}(2\zeta)^{\frac{1}{4}}
\end{align*}
holds for all $i\in[k]$ with at least a probability $1-\delta$,
where $\zeta=\mathit{\tilde\Theta}(\sqrt{k^5d/n})$.
\end{theorem}
We give a proof and the expression of $\zeta$ in Appendix~\ref{sec:thm:key:proof}. 

One potential limit of our proof strategy is that we may not correctly identify the sign of the ground truth neurons---i.e., our guarantee has the form $\hat\theta_{\sigma(i)}\approx\pm\theta^*_{i}$. However, we show in the next section that this caveat does not affect our bandit algorithm and analysis. Particularly, we show the difference between the estimated neural network $f_{\hat\Theta}$ and the true model $f_{\Theta^*}$ can be captured by a linear structure under sign misspecification. Thus, we can still run linear bandit algorithm to learn $f_{\theta^*}$ in an online manner.

\section{Algorithms for ReLU Bandits}
\label{sec:banditalgo}

Now, we describe our bandit algorithm and provide corresponding regret analysis. We begin with a simple case where we know the gap between the optimal action $x^*$ and the nearest neuron, and then provide a solution when this knowledge of gap is not assumed.

\subsection{Algorithm Design}

We first provide intuition for our design choices. The challenge of running a ReLU bandit algorithm is the nonlinearity of the ReLU neural network, which is due to the indicator function in the ReLU activations---i.e., in
\begin{align*}
\theta_i^\top x\cdot\mathbbm{1}(\theta_i^\top x\ge0),
\end{align*}
the first occurrence of $\theta_i$ is the linear contribution and the second occurrence is the contribution via the indicator function. The key of our design to tackle the nonlinearity is to first learn the indicator contribution, and then use a typical linear bandit algorithm to keep updating the model given the indicator function fixed. Particularly, 
once we have a sufficiently good estimate $\tilde\theta_i\approx\theta_i^*$, then our estimate of the indicator is exact:
\begin{align}
\label{eqn:indicatorequality}
\mathbbm{1}(\tilde\theta_i^\top x^*\ge0)=\mathbbm{1}(\theta_i^{*\top}x^*\ge0),
\quad \forall i\in[k].
\end{align}
Next, we can fix the value of the indicator function using $\tilde\theta_i$ and focus on learning the linear part. That is, we run a linear bandit algorithm with the value of $\theta_i = \tilde\theta_i$ in the indicator functions, but keep learning the value of $\theta_i$ in the linear part. In more detail, if \eqref{eqn:indicatorequality} holds, then we have
\begin{align}\label{def:rwd_bandit_eq}
\mathbb{E}[y]=f_{\Theta^*}(x)
&=\sum_{i\in[k]}(\mathbbm{1}(\tilde{\theta}_i^\top x\ge 0)x)^\top \theta_i^* \nonumber \\
&=
\underbrace{\begin{bmatrix}
\mathbbm{1}(\tilde\theta_1^\top x\ge 0)x \\
\mathbbm{1}(\tilde\theta_2^\top x\ge 0)x \\
\vdots \\
\mathbbm{1}(\tilde\theta_k^\top x\ge 0)x
\end{bmatrix}^\top}_{x^\dagger}
\underbrace{\begin{bmatrix}
\theta_1^* \\
\theta_2^* \\
\vdots \\
\theta_k^*
\end{bmatrix}}_{\theta^\dagger}.
\end{align}
Equivalently, we can run a linear bandit to learn $f_{\Theta^*}(x)$, however, with action being the term $x^\dagger\in\mathbb{R}^{dk}$ in \eqref{def:rwd_bandit_eq}, and parameter being the term $\theta^\dagger\in\mathbb{R}^{dk}$. The action $x^\dagger$ is a function of the original action $x$ and the estimated parameter $\tilde{\Theta}$, and the parameter to update $\theta^\dagger$ is a vectorization of the parameter $\Theta$ of interest. 

\textbf{Challenges.} Though the above two-stage design looks intuitive, there are still two challenges associated with \eqref{eqn:indicatorequality} and \eqref{def:rwd_bandit_eq}. On one hand, even if we have a relatively accurate estimate $\tilde\theta_i$ of $\theta_i^*$, \eqref{eqn:indicatorequality} might not hold for some action $x$ close to both of the estimate $\tilde\theta_i$ and the neuron $\theta_i^*$. On the other hand, our theoretical result in \S\ref{sec:relunetwork} only provides a guarantee on $\tilde\theta_i$ up to a sign flip, and hence still \eqref{eqn:indicatorequality} might not hold; thus, we also need to design our algorithm capturing this bias.

As suggested above, note that for any $x$ such that $\theta_i^{*\top}x\approx0$, even if $\tilde\theta_i\approx\theta_i^*$, it may still be the case that \eqref{eqn:indicatorequality} fails to hold. As a consequence, \eqref{def:rwd_bandit_eq} does not hold, and $f_{\Theta^*}(x)$ is not linear in $x^\dagger$ for $x$ close to any of $\theta_i^*$'s. In other words, the linear bandit is misspecified in some action region, so the algorithm may not converge if the optimal action lies in such regions. Thus, our regret bounds depend on the gap between the optimal action $x^*=\argmax_{x\in\mathcal{X}}f_{\Theta^*}(x)$ and the nearest hyperplane that is perpendicular to one of the neurons in the ground-truth neural network.
\begin{definition}
A ReLU neural network with parameters $\Theta^*$ has a $\nu_*$-gap for $\nu_*\in\mathbb{R}_{\ge0}$ if $|\theta_i^{*\top}x^*|\ge\nu_*$ for all $i\in[k]$.
\end{definition}
The following result ensures that a nontrivial gap $\nu_*>0$ always exists for our ReLU structure.
\begin{proposition}
\label{prop:optimalactiongap}
$\min_{i\in[k]}|\theta_i^{*\top}x^*|>0$ holds for the optimal action $x^*$.
\end{proposition}
We give a proof in Appendix~\ref{sec:prop:optimalactiongap:proof}. Note that our proof strategy relies on the ReLU structure where the weights of the second layer equal 1. In other words, any ReLU neural network structure as in \eqref{def:rwd_bandit} that satisfies our assumptions has a positive gap $\nu_* = \min_{i\in[k]}|\theta_i^{*\top}x^*|>0$. 

Given this gap $\nu_*$, as long as our estimate $\tilde\theta_i$ has an estimation error of $\theta_i^*$ smaller than $\nu_*/2$, i.e., $\|\tilde\theta_i - \theta_i^*\|\le \nu_*/2$, our bandit algorithm will be able to find the optimal action $x^*$ in the action space $\mathcal{X}(\tilde\Theta, \nu_*/2)$, where
\begin{align}\label{def:calX_nu}
\mathcal{X}(\Theta, \nu) &\coloneqq \{x\in\mathcal{X}\mid|\theta_i^{\top}x|\ge\nu, \forall i\in[k]\}.
\end{align}

Intuitively, the above claim holds for two reasons: (i) $x^*\in\mathcal{X}(\tilde\Theta, \nu_*/2)$, and (ii) the bandit model $f_{\Theta^*}(x)$ is linear in $x^\dagger$ for any $x\in\mathcal{X}(\tilde\Theta, \nu_*/2)$ (recall $x^\dagger$ is a function of $x$). 
For (i), it suffices to show that $\mathcal{X}(\Theta^*, \nu_*)\subseteq\mathcal{X}(\tilde\Theta, \nu_*/2)$, as our ReLU neural network in \eqref{def:relu} has a positive $\nu_*$-gap and hence $x^*\in\mathcal{X}(\Theta^*, \nu_*)$. To this end, for any $x\in\mathcal{X}(\Theta^*, \nu_*)$, if $\theta_i^{*\top}x>0$, we have
\begin{align}
\label{eqn:frozentheta}
\tilde\theta_i^\top x
&\ge\theta_i^{*\top}x-|\tilde\theta_i^\top x-\theta_i^{*\top}x| \nonumber\\
&\ge\theta_i^{*\top}x -\|\tilde\theta_i-\theta_i^*\|_2 \nonumber \\
&\ge\nu_*-\nu_*/2=\nu_*/2>0. 
\end{align}
Similarly, if $\theta_i^{*\top}x<0$, we have $\tilde\theta_i^\top x\le \nu_*/2$.
Next, to show (ii), it suffices to show that for any $x\in\mathcal{X}(\tilde\Theta, \nu_*/2)$, we have $\mathbbm{1}(\theta_i^{*\top} x\ge 0)=\mathbbm{1}(\tilde\theta_i^\top x\ge 0)$. Following a similar argument as \eqref{eqn:frozentheta}, we can show that $\theta_i^{*\top}x^*\ge0$ if $\tilde\theta_i^\top x\ge0$ (and similarly for $\le$). Thus, our plug-in indicator function is consistent with the true indicator function. 

The other challenge is the sign misidentification from our Theorem~\ref{thm:key}. Specifically, $\tilde\Theta$ is close to the true parameters $\Theta^*$ only up to signs. In the worst case, the values of the corresponding indicators $\mathbbm{1}(\tilde\theta_i^\top x\ge0)$ may differ from the true values $\mathbbm{1}(\theta_i^{*\top}x\ge0)$ when $\tilde\theta_i$ is close to $-\theta_i^*$ instead of $\theta_i^*$. In other words, the function $f_{\Theta^*}(x)$ will still be nonlinear of $x^\dagger$ and \eqref{def:rwd_bandit_eq} may no longer hold, leading to misspecification, when $\|\tilde\theta_i+\theta_i^*\|_2\le\nu_*/2$ instead of $\|\tilde\theta_i-\theta_i^*\|_2\le\nu_*/2$, even if we reduce our search region to $\mathcal{X}(\tilde\Theta, \nu_*/2)$

In order to correct this misspecification bias, we propose to add additional $k$ delicately designed linear components to \eqref{def:rwd_bandit_eq}. 
We show that the misspecification when $\|\tilde\theta_i+\theta_i^*\|_2\le\nu_*/2$ can be captured by a linear structure of $k$ additional transformed features of $x$, enabling the linear bandit algorithm to function again. In more detail, we can write
\begin{multline*}
f_{\Theta^*}(x)
=\sum_{i\in[k]}(\mathbbm{1}(\tilde{\theta}_i^\top x\ge 0)x)^\top \theta_i^*
\\
+ \sum_{i\in[k]} ((\frac{1}{2}-\mathbbm{1}(\tilde{\theta}_i^\top x\ge 0)x)^\top \\
\cdot \left(\frac{\mathbbm{1}(\theta_i^{*\top}x\ge 0)-\mathbbm{1}(\tilde{\theta}_i^\top x\ge 0)}{\frac{1}{2}-\mathbbm{1}(\tilde{\theta}_i^\top x\ge 0)}\theta_i^*\right).
\end{multline*}
Note that compared to \eqref{def:rwd_bandit_eq}, we have an additional second term that captures the misspecification; this term equals 0 for any $x\in\mathcal{X}(\tilde\Theta, \nu_*/2)$ when $\|\tilde\theta_i - \theta_i^*\|\le \nu_*/2$, as we have shown in \eqref{eqn:frozentheta}.
Similarly, when $\|\tilde\theta_i + \theta_i^*\|\le \nu_*/2$, we can show that $\mathbbm{1}(\tilde\theta_i^\top x\ge0)=1-\mathbbm{1}(\theta_i^{*\top} x\ge0)$---i.e., if $\theta_i^{*\top}x\gtreqless0$, we have $\tilde\theta_i^\top x\lessgtr0$ for any $x\in\mathcal{X}(\tilde\Theta, \nu_*/2)$. In other words, we have
\begin{align*}
\frac{\mathbbm{1}(\theta_i^{*\top}x\ge 0)-\mathbbm{1}(\tilde{\theta}_i^\top x\ge 0)}{\frac{1}{2}-\mathbbm{1}(\tilde{\theta}_i^\top x\ge 0)}
=\left\{
\begin{matrix}
0, & \text{if $\|\tilde\theta_i-\theta_i^*\|_2\le\frac{\nu_*}{2}$}\\
2, & \text{if $\|\tilde\theta_i+\theta_i^*\|_2\le\frac{\nu_*}{2}$}
\end{matrix}.
\right.
\end{align*}

Therefore, the true reward function $f_{\Theta^*}(x)$ in \eqref{def:rwd_bandit} is equivalent to
\begin{align}\label{eq:model_sign}
f_{\theta^\ddagger}(x^\ddagger)\coloneqq x^\ddagger(x, \tilde\Theta)^\top \theta^\ddagger(\Theta^*, \tilde\Theta),
\end{align}
where $x^\ddagger: \mathcal{X}\times\mathbb{R}^{k\times d}\to\mathbb{R}^{2kd}$ and $\theta^\ddagger:\mathbb{R}^{k\times d}\times\mathbb{R}^{k\times d}\to\mathbb{R}^{2kd}$ are two mappings with
\begin{align}
\label{def:xtheta_eq_ucb_x}
x^\ddagger(x, \tilde\Theta) &= 
\begin{bmatrix}
\mathbbm{1}(\tilde\theta_1^\top x\ge 0)x \\
\mathbbm{1}(\tilde\theta_2^\top x\ge 0)x \\
\vdots \\
\mathbbm{1}(\tilde\theta_k^\top x\ge 0)x \\
(\frac{1}{2}-\mathbbm{1}(\tilde{\theta}_1^\top x\ge 0))x \\
(\frac{1}{2}-\mathbbm{1}(\tilde{\theta}_2^\top x\ge 0))x \\
\vdots \\
(\frac{1}{2}-\mathbbm{1}(\tilde{\theta}_k^\top x\ge 0))x
\end{bmatrix},
\end{align}
and
\begin{align}
\label{def:xtheta_eq_ucb_tht}
\theta^\ddagger(\Theta^*, \tilde\Theta) &= 
\begin{bmatrix}
\theta_1^* \\
\theta_2^* \\
\vdots \\
\theta_k^*\\
2\theta_1^*\mathbbm{1}(\|\tilde\theta_1+\theta_1^*\|_2\le\frac{\nu_*}{2}) \\
2\theta_2^*\mathbbm{1}(\|\tilde\theta_2+\theta_2^*\|_2\le\frac{\nu_*}{2}) \\
\vdots \\
2\theta_k^*\mathbbm{1}(\|\tilde\theta_k+\theta_k^*\|_2\le\frac{\nu_*}{2})
\end{bmatrix}. 
\end{align}
In the following, we will use $x^\ddagger$ and $\theta^\ddagger$ to denote the two vectors in \eqref{def:xtheta_eq_ucb_x} and \eqref{def:xtheta_eq_ucb_tht} for simplicity whenever no ambiguity is raised.
The reward function in \eqref{eq:model_sign} has additional $k$ linear components compared to \eqref{def:rwd_bandit_eq}; $x^\ddagger$ builds upon $x^\dagger$ and contains additional $k$ features that captures the misspecification due to sign flip. Now, we will be able to run a linear bandit algorithm with $2kd$ features to learn $\theta^\ddagger$ in an online manner.

\subsection{OFU-ReLU Algorithm}\label{sec:exp_oful_alg}

For now, we assume $\nu_*$ is known throughout our design; we describe how to remove this assumption in \S\ref{sec:doublingtrick}. 

Our algorithm, which we name OFU-ReLU\footnote{``OFU" standards for optimism in the face of uncertainty \citep{abbasi2011improved}}, has two phases:
\begin{itemize}
\item \textbf{Exploration.} 
Randomly sample exploratory actions $x_t\sim p$ for $t_0$ time steps until our estimate $\tilde\Theta_{t_0}$ (i.e., $\tilde\Theta$ estimated using the first $t_0$ samples) satisfies
\begin{align*}
\min\{\|\tilde\theta_{t_0,i}-\theta_i^*\|_2,\|\tilde\theta_{t_0,i}+\theta_i^*\|_2\}\le\nu_*/2
\end{align*}
with high probability. 
\item \textbf{OFUL.} 
Run the OFUL algorithm~\citep{abbasi2011improved} to learn the true reward function $f_{\theta^\ddagger}(x^\ddagger)$ in \eqref{eq:model_sign}, which is linear in the parameter $\theta^\ddagger$ and features $x^\ddagger(x, \tilde\Theta_{t_0})$, over the region $x^\ddagger\in\mathcal{X}^\ddagger(\tilde\Theta_{t_0}, \nu^*/2)$, where
\begin{align}\label{def:calX_nu_ddagger}
\mathcal{X}^\ddagger(\Theta, \nu) \coloneqq \{x^\ddagger(x,\Theta)\mid x\in\mathcal{X}(\Theta,\nu)\}.
\end{align}
At each time period $t>t_0$, we follow OFUL, choosing arm $x_t^\ddagger = \argmax_{(x,\theta)\in\mathcal{X}^\ddagger(\tilde\Theta_{t_0},\nu^*/2)\times C_t(\lambda, Z_t)}x^{\top}\theta$ and observing reward $y_t$; the confidence ellipsoid $C_t(\lambda, Z_t)$ for the true parameter $\theta^\ddagger$ depends on a regularization hyperparameter $\lambda$ from OFUL, and can be computed using Theorem 2 in \cite{abbasi2011improved} with $S=\sqrt{5k}$ (note that $\|\theta^\ddagger\|\le \sqrt{5k}$) and all the data previously observed $Z_t = \{(x_\tau, y_\tau)\}_{\tau\in[t-1]}$.
\end{itemize}
We detail our algorithm in Algorithm~\ref{alg:eturelu}.
\begin{algorithm}[tb]
\begin{algorithmic}
\STATE \textbf{Input:} exploration length $t_0$, regularization parameter $\lambda$
\STATE Initialize $Z_0\gets\varnothing$
\FOR {$t\in[t_0]$}
\STATE Sample action $x_t\sim_{\text{i.i.d.}} p$
\STATE Take action $x_t$ and obtain reward $y_t$
\STATE $Z_t \leftarrow Z_{t-1} \cup \{(x_t, y_t)\}$
\ENDFOR
\STATE Compute $\tilde\Theta_{t_0}\gets\argmin_{\Theta} \hat{L}_S(\Theta;Z_{t_0})$
\FOR {$t \in (t_0+1,T]$}
\STATE Compute confidence ellipsoid $C_t(\lambda, Z_t)$ for $\theta^\ddagger$
\STATE $x_t^\ddagger \gets \argmax_{(x,\theta)\in\mathcal{X}^\ddagger(\tilde\Theta_{t_0},\nu^*/2)\times C_t(\lambda, Z_t)}x^{\top}\theta$
\STATE Play $x_t$ with $x^\ddagger(x_t, \tilde\Theta_{t_0}) = x_t^\ddagger$ and obtain reward $y_t$
\STATE $Z_t\gets Z_{t-1}\cup\{(x_t,y_t)\}$
\ENDFOR
\end{algorithmic}
\caption{OFU-ReLU}
\label{alg:eturelu}
\end{algorithm}
Given our design, we can obtain a regret bound scaling as $\tilde{O}(kd\sqrt{T})$; particularly, we control the parameter estimation error of $\tilde\Theta_{t_0}$ using Theorem \ref{thm:key} in \S\ref{sec:relunetwork}, and analyze the regret of OFUL stage using Theorem 3 in \cite{abbasi2011improved}. We provide a proof sketch below.
\begin{theorem}\label{thm:algo1_reg}
The cumulative regret of Algorithm~\ref{alg:eturelu} satisfies
\begin{align*}
R_T=\tilde{O}\left(k^{14}d^3(1/\nu_*^8\lor d^4)+kd\sqrt{T}\right).
\end{align*}
\end{theorem}

\begin{proof}[Proof Sketch]
Suppose the exploration stage ends at time $t_0$. We have
\begin{align*}
\min\left\{\|\tilde\theta_i-\theta_i^*\|_2,\|\tilde\theta_i+\theta_i^*\|_2\right\}\le\nu_*/2, \quad \forall i\in[k]
\end{align*}
with probability at least $1-\delta/2$ by choosing $t_0$ large enough according to Theorem~\ref{thm:key}. In particular, it suffices to take $\delta = 1/\sqrt{T}$ and 
\begin{align}\label{eq:t0choice}
t_0 = \mathit{\tilde\Omega}(k^{13}d^3(1/\nu_*^8\lor d^4)).
\end{align}

Now we analyze the regret of our OFU-ReLU algorithm in three cases.
First, at each time $t$, the per-period regret can be bounded trivially by 
\begin{align*}
r_t
\le (\sum_{i\in[k]} \|\theta_i^*\|_2)(\|x^*\|_2+\|x_t\|_2) 
\le 2k.
\end{align*}
Therefore, the regret during the exploration phase is upper bounded by $\tilde{O}(k^{14}d^3(1/\nu_*^8\lor d^4))$.

In the second stage, we run OFUL to find the optimal policy for the linear function $f_{\theta^\ddagger}(x^\ddagger)$ given our forced-sample estimate $\tilde\Theta_{t_0}$. Applying the regret bound in Theorem 3 of \cite{abbasi2011improved} gives the regret bound in our second stage to be $\tilde{O}(kd\sqrt{T})$ with at least a probability of $1-\delta/2$.

Finally, with a small probability $\delta=1/\sqrt{T}$, we would have linear regret scaling as $2kT$; thus, the expected regret in this case is bounded by $2kT\delta =O(k\sqrt{T})$. Our claim then follows. 
\end{proof}
We provide a detailed proof in Appendix \ref{app:thm_bandit}. Theorem \ref{thm:algo1_reg} shows that our algorithm obtains a $\tilde{O}(kd\sqrt{T})$ regret guarantee as long as the time horizon $T$ is large enough.

\subsection{OFU-ReLU+ Algorithm}
\label{sec:doublingtrick}

Algorithm~\ref{alg:eturelu} requires the knowledge of the gap $\nu_*$, which is typically unknown.
We can remove this assumption based on a batching strategy; we provide a schematic representation of our algorithm in Figure \ref{fig:batch_ill}. Algorithm~\ref{alg:etudrelu} summarizes our algorithm OFU-ReLU+ based on this insight.

At a high level, we split the entire time horizon $T$ into $M$ increasing batches with a grid $0=T_0\le T_1\le \cdots \le T_M=T$. Each batch $i\in[M-1]$, i.e., $t\in(T_{i-1}, T_{i}]$, satisfies $a(T_i - T_{i-1})=T_{i+1} - T_i$---i.e., $T_i=(a^i-1)T_1$ for some constant $a>1$ and $T_1>1$. Note that 
\begin{align}\label{eq:valM}
M=\left\lceil\frac{\log(T/T_1+1)}{\log(a)}\right\rceil.
\end{align}
For each batch, we take a fixed guess of $\nu_*$; we reduce this guess geometrically from one batch to the next. Specifically, let $\nu_0$ be our initial guess of $\nu_*$ at $t=0$; then, the guess $\nu_i$ for batch $i$ is $\nu_i = \nu_0/b^i$ for some constant $b>1$. Our $\nu_i$ will become sufficiently small that $\nu_i\le\nu_*$ for batch $i > \log(\nu_0/\nu_*)/\log(b)$. Our regret analysis in \S\ref{sec:exp_oful_alg} can then be applied from that batch onwards. 

\begin{figure}[t]
  \centering
  \includegraphics[width=\columnwidth]{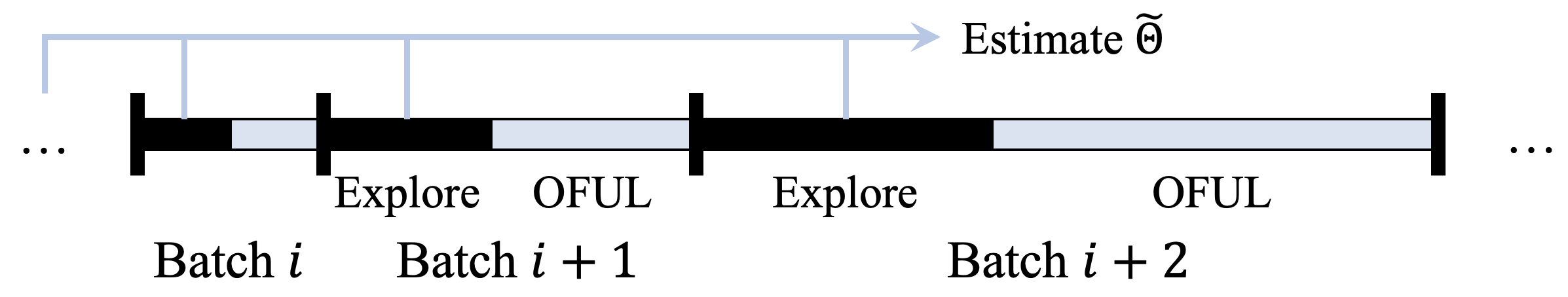}
\caption{Schematic representation of OFU-ReLU+.}
\label{fig:batch_ill}
\vskip -0.2in
\end{figure}

In detail, within each batch $i$, we have an exploration phase and a following OFUL phase as in \S\ref{sec:exp_oful_alg}. We initialize the batch with the exploration period as $t\in(T_{i-1}, T_{i-1}+t_{0, i}]$. Different from the conventional batching strategies \citep[see, e.g.,][]{golrezaei2019dynamic,luo2022contextual}, we estimate $\tilde\Theta$ and run OFUL using samples from all previous batches without discarding data from previous batches. 
Particularly, we estimate $\tilde\Theta$ based on the accumulative samples from all the previous exploration phases $t\in\cup_{j\in[i]}(T_{j-1}, T_{j-1}+t_{0, j}]$ so that the estimation error satisfies
\begin{align}
\label{eq:tildesample}
\min\{\|\tilde\theta_i-\theta_i^*\|_2, \|\tilde\theta_i+\theta_i^*\|_2\}\le\nu_i/2,
\quad \forall i\in[k].
\end{align}
Define a function $t_0$\footnote{Note that $t_0$ in \S\ref{sec:exp_oful_alg} is a function of $\nu_*$. Here we abuse the notation $t_0$ slightly.} of an arbitrary $\nu$ as
\begin{align}\label{eq:rt0_def}
t_0(\nu) = \mathit{\tilde\Theta}\left(k^{13}d^3(1/\nu^8\lor d^4)\right)
\end{align}
according to \eqref{eq:t0choice} (detailed expression in Appendix \ref{app:thm_bandit}).
Then, it suffices to have $t_{0, i} = t_0(\nu_i)-t_0(\nu_{i-1})$ according to \eqref{eq:t0choice}. Let our estimate be $\tilde\Theta_{t_0(\nu_i)}$ (i.e., $\tilde\Theta$ estimated using all $t_0(\nu_i)$ random samples from previous batches). 

In the rest of the batch $i$, we run OFUL; at each time step $t\in(T_{i-1}+t_{0, i}, T_{i}]$, we choose arm $x_t^\ddagger = \argmax_{(x,\theta)\in\mathcal{X}^\ddagger(\tilde\Theta_{t_0(\nu_i)},\nu_i/2)\times C_t(\lambda, Z_t)}x^{\top}\theta$ and observing reward $y_t$. Again, we compute the $C_t(\lambda, Z_t)$ based on \cite{abbasi2011improved} with $S=\sqrt{5k}$ and all the data observed $Z_t=\{(x_\tau, y_\tau)\}_{\tau\in[t-1]}$. The features are $x^\ddagger(x_\tau, \tilde\Theta_{t_0(\nu_i)})$ and the parameter to estimate is $\theta^\ddagger(\Theta^*, \tilde\Theta_{t_0(\nu_i)})$. Note that the features and the parameter are invariant of $\tilde\Theta_{t_0(\nu_i)}$ once $\nu_i\le\nu_*$; thus, our estimate becomes consistent and OFUL is valid from then on.

To bound the regret, we decompose the whole time horizon into three parts and analyze them respectively: 
\begin{enumerate}[(i).]
\item All batch $i$ satisfying $\nu_i>\nu_*$,
\item $t\in(T_{i-1}, T_{i-1}+t_{0,i}]$ for all batch $i$ with $\nu_i\le\nu_*$,
\item $t\in(T_{i-1}+t_{0,i}, T_i]$ for all batch $i$ with $\nu_i\le\nu_*$.
\end{enumerate}
The regret in (i) is independent of $T$ and is based on our choices of $a$, $b$, $T_1$ and $\nu_0$. Our analysis in \S\ref{sec:exp_oful_alg} can be applied similarly to analyze the exploration regret in (ii) and the regret (iii) from OFUL. 
However, for (ii), since $\nu_i$ decreases over time, the regret per batch grows over time. We show in the following theorem that this regret scales as a polynomial term of $T$ of which the order depends on the relative size between $a$ and $b$, and can be chosen to be asymptotically less than $O(\sqrt{T})$.

\begin{algorithm}[t]
\begin{algorithmic}
\STATE \textbf{Input:} regularization parameter $\lambda$, parameters $\nu_0, T_1, a, b$
\STATE Initialize $Z_0\gets\varnothing$, $E_0\gets\varnothing$
\FOR {$i\in[M]$}
\STATE $\nu_i \gets \nu_{i-1}/b$, $t_{0,i}\gets t_0(\nu_i)-t_0(\nu_{i-1})$
\FOR {$t\in(T_{i-1}, T_{i-1}+t_{0,i}]$}
\STATE Sample action $x_t\sim_{\text{i.i.d.}} p$
\STATE Take action $x_t$ and obtain reward $y_t$
\STATE $E_t \gets E_{t-1} \cup \{(x_t, y_t)\}$, $Z_t\gets Z_{t-1}\cup\{(x_t,y_t)\}$
\ENDFOR
\STATE Compute $\tilde\Theta_{t_0(\nu_i)}\gets\argmin_{\Theta}\hat{L}_S(\Theta;E_{T_i+t_{0,i}})$
\FOR {$t\in(T_{i-1}+t_{0,i}, T_i]$}
\STATE Compute confidence ellipsoid $C_t(\lambda, Z_t)$ for $\theta^\ddagger$
\STATE $x_t^\ddagger \gets \argmax_{(x,\theta)\in\mathcal{X}^\ddagger(\tilde\Theta_{t_0(\nu_i)},\nu_i/2)\times C_t(\lambda, Z_t)}x^{\top}\theta$
\STATE Play $x_t$ with $x^\ddagger(x_t, \tilde\Theta_{t_0(\nu_i)}) = x_t^\ddagger$ and obtain $y_t$
\STATE $Z_t\gets Z_{t-1}\cup\{(x_t,y_t)\}$
\ENDFOR
\STATE $T_{i+1} \gets (a^{i+1}-1)T_1$
\ENDFOR
\end{algorithmic}
\caption{OFU-ReLU+}
\label{alg:etudrelu}
\end{algorithm}

\begin{theorem}\label{thm:algo2_reg}
The cumulative regret of Algorithm~\ref{alg:etudrelu} has 
\begin{align*}
R_T =
\tilde{O}\left(k^{14}d^7+k^{14}d^3T^{8\frac{\log(b)}{\log(a)}}+kd\sqrt{T}\right).
\end{align*}
\end{theorem}
\begin{proof}[Proof Sketch]
We bound the regret for the three cases above respectively. First, in case (i), we have $i \le \log(\nu_0/\nu_*)/\log(b)$. Recall that the per-period regret can be trivially bounded by $2k$. Thus, the regret in this case is upper bounded by 
\begin{align*}
2k (a^{\log(\nu_0/\nu_*)/\log(b)}-1) T_1 \le 2k(\nu_0/\nu_*)^{\frac{\log(a)}{\log(b)}}T_1.
\end{align*}

Similarly, regret in case (ii) can also be bounded by
\begin{align*}
2k\sum_{i=\lceil\frac{\log(\nu_0/\nu_*)}{\log(b)}\rceil}^{M-1} t_{0,i} =& \tilde{O}\left(k^{14}d^7+k^{14}d^3T^{8\frac{\log(b)}{\log(a)}}\right),
\end{align*}
where we use the definition of $t_{0,i}$, that of $t_0(\nu)$ in \eqref{eq:rt0_def}, and the value of $M$ in \eqref{eq:valM}.

Next, we calculate the regret of running OFUL in case (iii), which is upper bounded again by $\tilde{O}(kd\sqrt{T})$ following the proof strategy of Theorem 3 in \cite{abbasi2011improved}, similar to our proof for OFU-ReLU.

Finally, with a union bound, there's a small probability $M/\sqrt{T}$ that we will have a linear regret since the above analysis holds only with high probability. We can show the regret of this part scales as $\tilde{O}(k\sqrt{T})$. Combining all the above gives our final result.
\end{proof}

We provide a detailed proof in Appendix \ref{app:pf_ofureluplus}.
Compared with Theorem \ref{thm:algo1_reg}, here we gain an additional $T^{8\frac{\log(b)}{\log(a)}}$ dependence due to the increasing difficulty of learning the unknown gap $\nu_*$. 
Theorem \ref{thm:algo2_reg} implies that, as long as our choices of the multipliers $a$ and $b$ satisfy $8\log(b)/\log(a)\le 1/2$, we recover a $\tilde{O}(\sqrt{T})$ regret dependence---i.e., when the length of exploration period grows slower than the batch time horizon. for instance, taking $a=2$ and $b=2^{1/32}$, we obtain
\begin{align*}
R_T = \tilde{O}\left(
k^{14}d^7+k^{14}d^3T^{1/4}+kd\sqrt{T}\right).
\end{align*}
Finally, if the time horizon $T$ is at least a polynomial term of $k$ and $d$, then as before, we recover an $\tilde{O}(kd\sqrt{T})$ regret guarantee.

\section{Experiments}

We compare our algorithm OFU-ReLU with several benchmarks, including OFUL \citep{abbasi2011improved}, which assumes the true model is linear and introduces misspecification errors, and three different versions of NeuralUCB \citep{zhou2020neural}, i.e., NeuralUCB-F, NeuralUCB-T and NeuralUCB-TW. Particularly, NeuralUCB-F follows the setup in \S 7.1 of \cite{zhou2020neural} with $m=20$ neurons and two layers; NeuralUCB-T assumes the knowledge of the neural network structure of the true reward, i.e., $m=k$ neurons and one layer; and NeuralUCB-TW inherits the structure from NeuralUCB-T but expands the layer size into $m=2k$. 

We consider the true model of a ReLU structure as in \eqref{def:rwd_bandit}, with multiple settings presented in Figure~\ref{fig:simulation}. The parameter $\Theta^*$ is randomly sampled from the sphere $\|\theta_i\|=1$ for $i\in[k]$. The noise follows a normal distribution $N(0, 0.01)$. We randomly draw $1,000$ arms from the unit sphere in each round $t$ and choose an optimal arm from this arm set. Note that with a discretized arm set, our claim of a nontrivial gap $\nu_*$ always holds. For OFUL and OFU-ReLU, we use the theoretically suggested confidence ellipsoid for UCB. Since we do not know the gap $\nu_*$, we set the length of exploration phase for OFU-ReLU to be 20 for our method. We tune the hyperparameters $\lambda$ for all the methods. 

Figure \ref{fig:simulation} shows the performance of our bandit algorithm versus the other benchmarks with a 95\% confidence interval. We find our algorithm significantly outperforms all the other benchmarks. OFUL assumes a linear reward model structure and thus incurs large regret due to the misspecification error. All NeuralUCB benchmarks use gradient descent to learn the model structure over time and thus take long time to converge in general. Note that even as a fair comparison with NeuralUCB-T, where the true network structure is given, our method is still significantly better in terms of regret. It is worth noting that our method takes only $20$ time steps to converges in a time horizon of $T=1,000$, while NeuralUCB algorithms generally take a long time to converge (e.g., \cite{zhou2020neural} consider a longer horizon $T\approx 10,000$ in all their experiments). Our empirical results complement our theoretical analysis and suggest the efficiency of our algorithm in practice, especially in a short time horizon, despite a theoretically long exploration phase due to our parameter estimation error bound.

\begin{figure}[t]
\centering
\begin{subfigure}[b]{0.4\textwidth}
  \centering
  \includegraphics[width=\textwidth]{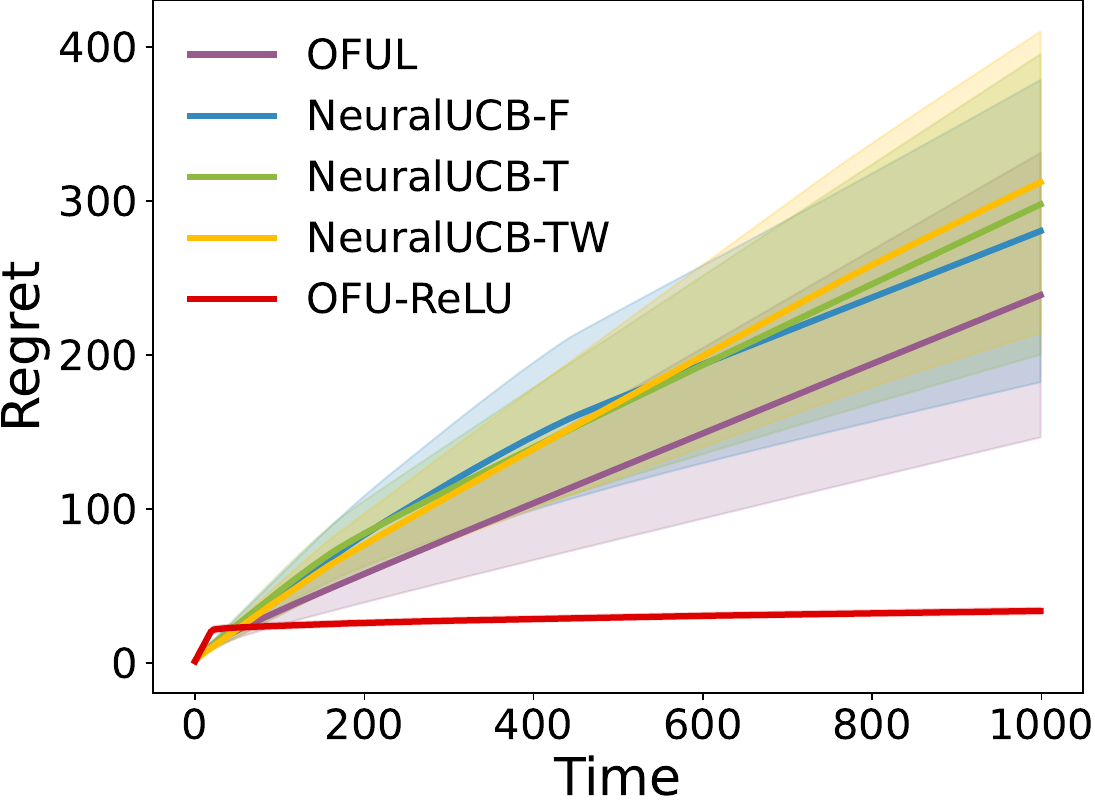}
  \caption{$d=2,k=3$}
\end{subfigure}
\begin{subfigure}[b]{0.4\textwidth}
  \centering
  \includegraphics[width=\textwidth]{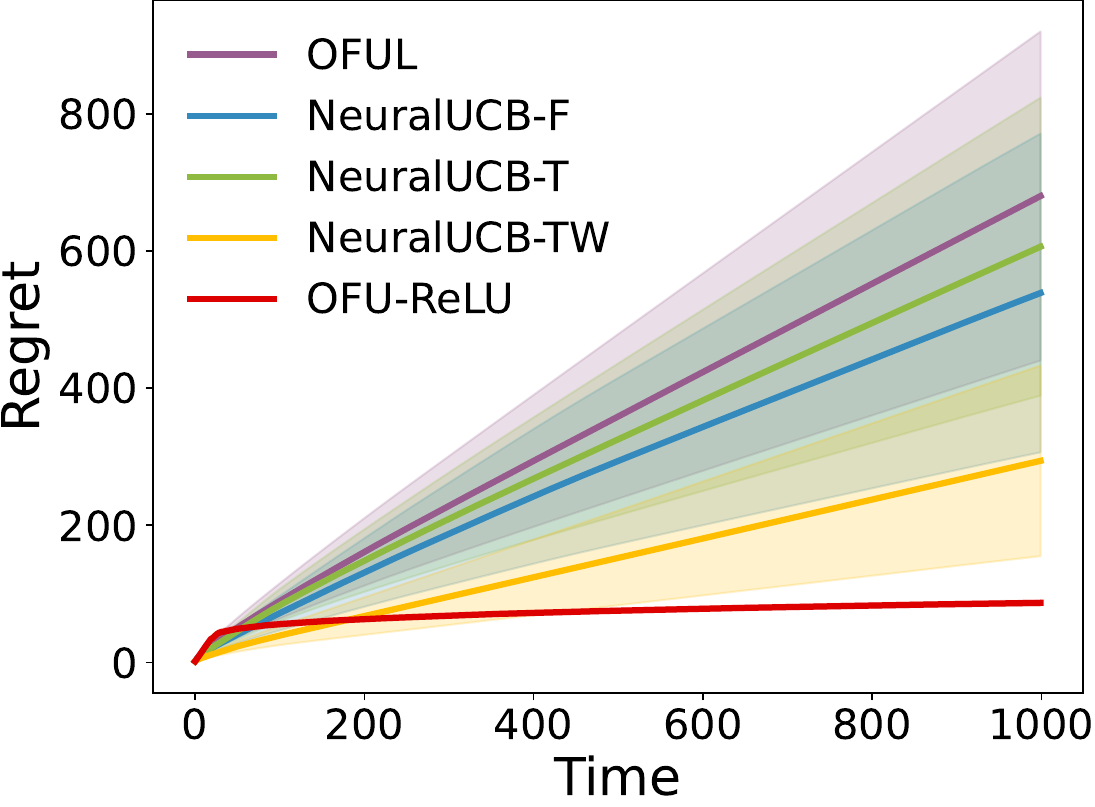}
  \caption{$d=2,k=10$}
\end{subfigure}
\caption{Cumulative regret of a time horizon $T=1,000$ over 50 trials with 95\% confidence interval.}
\label{fig:simulation}
\vskip -0.2in
\end{figure}

\section{Conclusion}

We analyze a bandit problem with the reward given by one-layer ReLU neural network structure, and propose algorithms that can provide a regret bound of $\tilde{O}(\sqrt{T})$. To the best of our knowledge, our work is the first to obtain such regret guarantee for bandit learning with neural networks (without an effective dimension assumption). Furthermore, we demonstrate the efficiency of our algorithm in a synthetic experiment, which suggests its practical potentials. We believe both our theoretical and empirical results provide the first insight into an efficient design of bandit algorithms based on ReLU neural network structures. 

We conclude by providing directions for future research. Due to the complexity of the problem, we tailor our focus to the one-layer ReLU activations. A natural extension is to generalize our result to piecewise linear activation functions. It is more challenging to explore whether our insight can be generalized to bandit problems with more complex activation functions or multiple-layer architectures.

\section*{Impact Statement}

This paper presents work whose goal is to advance the field of Machine Learning. There are many potential societal consequences of our work, none which we feel must be specifically highlighted here.

\bibliographystyle{icml2024}
\bibliography{ref}

\newpage
\appendix

\onecolumn

\section{Proof of Proposition~\ref{prop:key}}
\label{app:pf_prop}

We give the proof of Proposition~\ref{prop:key}, followed by proofs of the lemmas used in this proof.

\subsection{Intuition}

We illustrate our proof strategy in Figure \ref{fig:proofsketch}.
We first define the following set
\begin{align*}
\mathcal{J}^\alpha_i=\{l\in[k]\mid\|\theta_{l}-\theta_i^*\|_2\le\alpha, ~\text{or}~ \|\theta_{l}+\theta_i^*\|_2\le\alpha\}
\end{align*}
for all $i\in[k]$. It suffices to prove that $\mathcal{J}^\alpha_i$ is a singleton set for every $i\in[k]$ in order to prove Proposition~\ref{prop:key}. Note that when $\alpha\le\alpha_0/2$, $\mathcal{J}^\alpha_i$'s are disjoint (i.e., $\mathcal{J}^\alpha_i\cap\mathcal{J}^\alpha_{i'}=\varnothing$ for any $i,i'\in[k], i\ne i'$); in particular, if $i''\in \mathcal{J}^\alpha_i$, then for any other $i'\in[k]$, we have
\begin{align*}
\|\theta_{i''}\pm\theta_{i'}^*\|_2
\ge\|\theta_i^*-\theta_{i'}^*\|_2-\|\theta_{i''}-\theta_i^*\|_2\ge\alpha_0-\alpha
\ge\alpha,
\end{align*}
that is, $i''\not\in J^\alpha_{i'}$, as claimed, where we use Assumption~\ref{assump:neuronskip}. As a consequence, it suffices to show that $\mathcal{J}^\alpha_i\neq\varnothing$ for every $i\in[k]$. 
To this end, we prove its contrapositive---i.e., if there exists $j\in[k]$ such that $\mathcal{J}^\alpha_j=\varnothing$, then
\begin{align*}
L_{X_j}(\Theta)\ge\eta=\frac{\epsilon^2(1-d\epsilon^2/2)|S^{d-2}|}{8}-\frac{k\epsilon^3|S^{d-3}|}{2\alpha}-(6k+3\sigma)d\epsilon^3|S^{d-2}|.
\end{align*}

Intuitively, if $\theta_j^*$ does not have a matching neuron $\theta_j$ (up to a sign flip) in $\mathcal{J}_j^\alpha$, then we can show that $g(\theta^\top x)$ is linear for any
\begin{align*}
\theta \in \bar\Theta_{\lnot j}^* \coloneqq\{\theta_i\}_{i\in[k]}\cup\{\theta_i^*\}_{i\in[k], i\ne j}
\end{align*}
except $\theta^*_j$ on majority of the strip $X_j'$ (Figure \ref{fig:proofsketch} (b), formally defined in \eqref{def:xjprime}), which is a close approximation of $X_j$. Therefore, $f_\Theta(x)-f_{\Theta^*}(x)$ can be additively decomposed into a linear term plus $g(\theta_j^{*\top}x)$. Besides, we prove that any linear function cannot approximate $g(\theta_j^{*\top}x)$ well on $X_j'$ (Figure \ref{fig:proofsketch} (c)). 
As a result, given the definition of $L_{X_j}(\Theta)$, we can show that $L_{X_j}(\Theta)$ is lower-bounded.

\begin{figure*}[htbp]
\centering
\begin{subfigure}[b]{0.3\textwidth}
  \centering
  \includegraphics[width=\textwidth]
  {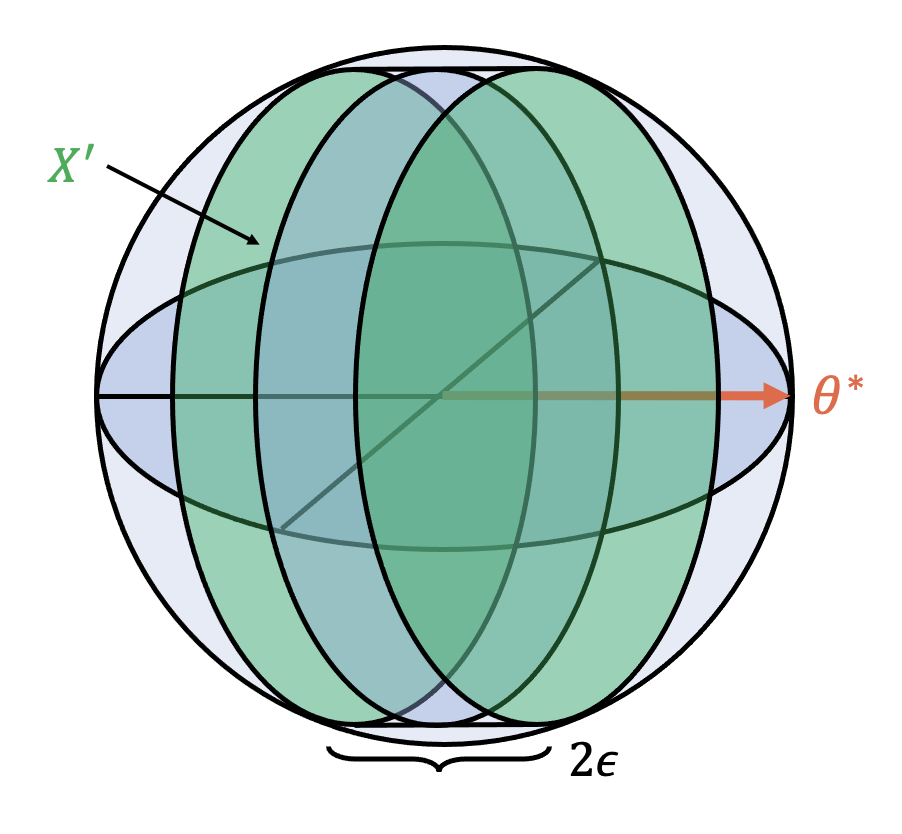}
  \caption{}
\end{subfigure}
\begin{subfigure}[b]{0.3\textwidth}
  \centering
  \includegraphics[width=\textwidth]
  {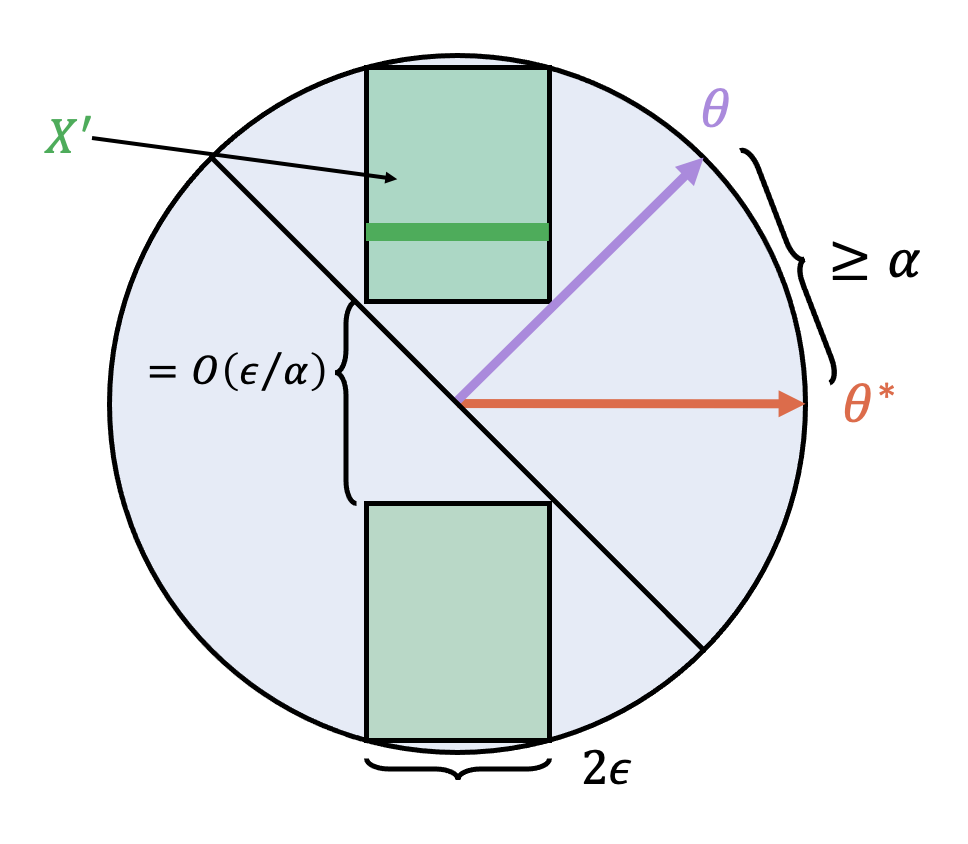}
  \caption{}
\end{subfigure}
\begin{subfigure}[b]{0.3\textwidth}
  \centering
  \includegraphics[width=\textwidth]
  {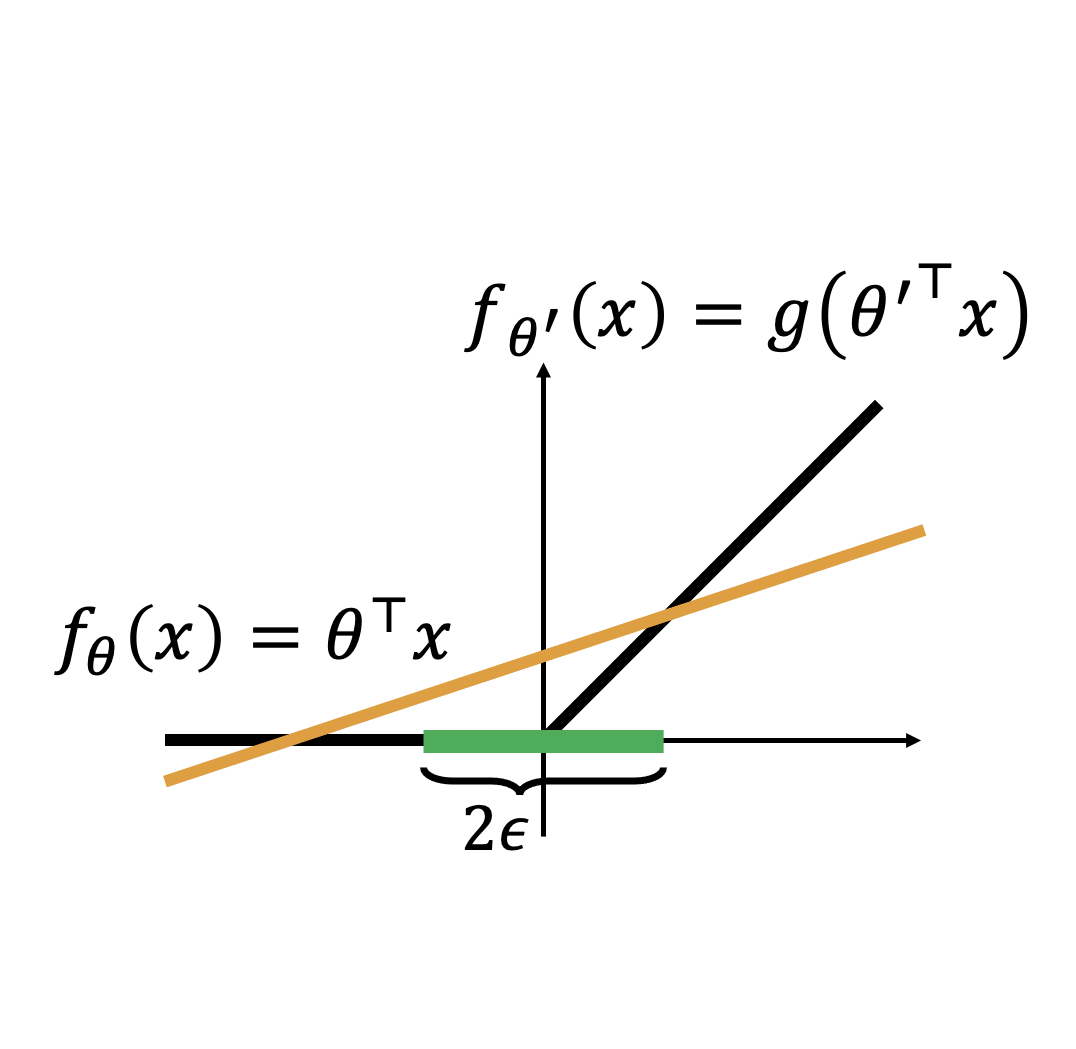}
  \caption{}
\end{subfigure}
\caption{Illustrations for proof sketch of the estimation error for ReLU Neural Networks. (a) The region $X'$ is the cylinder with caps consisting of the two green circles and radius $2\epsilon$. (b) Projected version of subfigure (a). (c) The green region is $X'$ with a section of length $O(\epsilon/\alpha)$ cut out.}
\label{fig:proofsketch}
\end{figure*}

\subsection{Proof of Proposition~\ref{prop:key}}
We list the details of our proof strategy in this section. We introduce additional notation that we will use in the following proof. With a slight abuse of notation, for any vector $x\in\mathbb{R}^d$ Let $x_i$ denote the $i^{th}$ element of $x$ and $x_{i:j}\in\mathbb{R}^{j-i+1}$ the subvector of $x$ consisting of the $i^{th}$ to $j^{th}$ elements.

\paragraph{Step 1.}
Note that $X_j$ is a slice of the sphere $\mathcal{X}=S^{d-1}$. To simplify our following analysis, we approximate $X_j$ using a cylinder $X_j'\subseteq\mathbb{R}^d$---i.e., without loss of generality assuming that for one specific $j\in[k]$
\begin{align}\label{eq:wlog_thetajs}
\theta_j^*=\begin{bmatrix}1&0&\cdots&0\end{bmatrix}^\top,
\end{align}
then 
\begin{align*}
X_j'=[-\epsilon,\epsilon]\times Z=\{x\in\mathbb{R}^d\mid|x_1|\le\epsilon, x_{2:d}\in Z\},
\quad\text{where}\quad
Z=S^{d-2}\left(\sqrt{1-\epsilon^2}\right),
\end{align*}
or equivalently
\begin{align}
\label{def:xjprime}
X_j'=\{\phi(x)\mid x\in X_j\}, \quad \text{where} \quad \phi(x)=\begin{bmatrix}x_1&\sqrt{\frac{1-\epsilon^2}{1-x_1^2}}x_2& \cdots &\sqrt{\frac{1-\epsilon^2}{1-x_1^2}}x_d\end{bmatrix}^\top.
\end{align}
This region is visualized in Figure~\ref{fig:proofsketch} (a); its projection to two dimensions is shown in Figure~\ref{fig:proofsketch} (b). We show that the loss restricted to $X_j'$ is approximately equal to the loss restricted to $X_j$:
\begin{lemma}
\label{lem:xapprox}
Given $X_j'$, it holds that
\begin{align*}
&\left|\int_{X_j}|f_\Theta(x)-f_{\Theta^*}(x)|dx-\int_{X_j'}|f_\Theta(x)-f_{\Theta^*}(x)|dx\right|\le 6kd\epsilon^3|S^{d-2}|.
\end{align*}
\end{lemma}
The proof is provided in Appendix~\ref{sec:lem:xapprox:proof}.

\paragraph{Step 2.}
Next, we decompose $X_j'$ into strips---i.e.,
\begin{align*}
X_j'=\bigcup_{z\in Z}X^z,
\quad\text{where}\quad
X^z=[-\epsilon,\epsilon]\times\{z\}.
\end{align*}
Note that $X^z$ is a one-dimensional manifold; one such strip is shown as the horizontal green line in Figure~\ref{fig:proofsketch} (b). Our strategy is to lower bound the loss restricted to each of these strips, after which we can integrate over them to obtain a lower bound on the overall loss. With a slight abuse of notation, we further define
\begin{align*}
f_{\Theta_{\lnot j}}=\sum_{i\in[k], i\ne j}g(\theta_i^\top x).
\end{align*}
Precisely, we provide a lower bound of the loss for those $z\in Z$ where $f_\Theta$ and $f_{\Theta_{\lnot j}^*}$ is linear on $X^z$; in particular, for such $z$, we have
\begin{align*}
X^z\cap \mathcal{V}^{\bar\Theta_{\lnot j}^*} = \varnothing,
\quad\text{where}\quad
\mathcal{V}^{\bar\Theta_{\lnot j}^*}=\bigcup_{\theta\in\bar\Theta_{\lnot j}^*}\{x\in X_j'\mid\theta^\top x=0\}.
\end{align*}
Note that $\mathcal{V}^{\bar\Theta_{\lnot j}^*}$ is the boundary at which one of the ReLUs, i.e., $g(\theta^\top x)$ with $\theta \in \bar\Theta_{\lnot j}^*$, transitions from inactive to active. If $X^z$ does not intersect $\mathcal{V}^{\bar\Theta_{\lnot j}^*}$, then $\theta^\top x\neq0$ on $X^z$ for all $\theta \in \bar\Theta_{\lnot j}^*$ and, hence, $f_\Theta-f_{\Theta_{\lnot j}^*}$ must be linear on such $X^z$. 
In the following, we show that such $z$'s make up a large proportion of $Z$; equivalently, we show that the following subset is small:
\begin{align}
\label{def:ZTheta}
Z^{\bar\Theta_{\lnot j}^*}=\bigcup_{\theta\in\bar\Theta_{\lnot j}^*} Z^{\theta},
\quad \text{where} \quad 
Z^{\theta}=\{z\in Z\mid\exists x_1\in[-\epsilon,\epsilon],\,\theta^\top(x_1\circ z)=0\},
\end{align}
where $x_1\circ z\coloneqq\begin{bmatrix}x_1&z_1&\cdots&z_{d-1}\end{bmatrix}^\top$.
\begin{lemma}
\label{lem:linearregion}
For any $\theta\in\bar\Theta_{\lnot j}^*$, we have
\begin{align*}
|Z^{\theta}|\le\frac{2\epsilon}{\alpha}|S^{d-3}|.
\end{align*}
\end{lemma}
The proof is given in Appendix~\ref{sec:lem:relulinear:proof}. This result is illustrated in Figure~\ref{fig:proofsketch} (b); the set of $X^z$ for which $z\in Z^{\bar\Theta_{\lnot j}^*}$ (which has size $O(\epsilon/\alpha)$) has been removed from $X_j'$. Note that as $\alpha$ becomes larger, the size of $Z^{\bar\Theta_{\lnot j}^*}$ becomes smaller.

\paragraph{Step 3.}
Next, for $z\in Z \setminus Z^{\bar\Theta_{\lnot j}^*}$, we lower bound the loss on $X^z$. Remember the loss is
\begin{align*}
|f_\Theta(x)-f_{\Theta^*}(x)| = |(f_\Theta(x)-f_{\Theta_{\lnot j}^*}(x)) - g(\theta_j^{*\top}x)|,
\end{align*}
where we argue in Step 2 that on such $X_z$ the first term is linear. Therefore, we can lower bound the loss using the following lemma:
\begin{lemma}
\label{lem:linearbound}
For any $\beta_0,\beta_1\in\mathbb{R}$, we have
\begin{align*}
\int_{-\epsilon}^\epsilon|(\beta_0+\beta_1w)-g(w)|dw\ge\frac{\epsilon^2}{8}.
\end{align*}
\end{lemma}
We provide the proof in Appendix~\ref{sec:lem:linearbound:proof}. Since our loss is the mean absolute error, this result follows from a geometric argument. Intuitively, as illustrated in Figure~\ref{fig:proofsketch} (c), there is a triangular gap between $f_\Theta-f_{\Theta^*_{\lnot j}}$ (which is linear---i.e., $f_\Theta(x)-f_{\Theta^*_{\lnot j}}(x)=\beta^\top x$ for some $\beta\in\mathbb{R}^d$) and $g(\theta_j^{*\top}x)$ on $X^z$; this gap equals the loss, and it cannot be reduced regardless of the value of $\beta$.

\paragraph{Step 4.}
Finally, the proof of Proposition~\ref{prop:key} consists of integrating the lower bound from Step 3 over $z\in Z \setminus Z^{\bar\Theta_{\lnot j}^*}$ to obtain a lower bound on $L_{X_j}(\Theta)$.

\begin{proof}
First, note that
\begin{align}\label{eq:ineqint_xjp}
\int_{X_j'}|f_\Theta(x)-f_{\Theta^*}(x)|dx
&\ge \int_Z\int_{-\epsilon}^\epsilon|f_\Theta(x_1\circ z)-f_{\Theta^*}(x_1\circ z)|dx_1dz \nonumber \\
&\ge\int_{Z \setminus Z^{\bar\Theta_{\lnot j}^*}}\int_{-\epsilon}^\epsilon|f_\Theta(x_1\circ z)-f_{\Theta^*}(x_1\circ z)|dx_1dz \nonumber \\
&=\int_{Z \setminus Z^{\bar\Theta_{\lnot j}^*}}\int_{-\epsilon}^\epsilon|(f_\Theta(x_1\circ z)-f_{\Theta_{\lnot j}^*}(x_1\circ z)) - g(\theta_j^{*\top}(x_1\circ z))|dx_1dz.
\end{align}
Remember for any given $z\in Z \setminus Z^{\bar\Theta_{\lnot j}^*}$, the first term $f_\Theta(x_1\circ z) - f_{\Theta_{\lnot j}^*}(x_1\circ z)$ is linear in $x_1 \circ z$, i.e., there exists a parameter $\tilde\theta_j$ such that $\tilde\theta_j^\top (x_1\circ z) = f_\Theta(x_1\circ z) - f_{\Theta_{\lnot j}^*}(x_1\circ z)$. Without loss of generality, we can modify the coordinate system so that 
\begin{align*}
\tilde\theta_j=\begin{bmatrix}t_1&t_2&0&\cdots&0\end{bmatrix}^\top
\end{align*}
without affecting $\theta_j^*$ in \eqref{eq:wlog_thetajs}. Then, we have
\begin{align*}
\int_{-\epsilon}^\epsilon|(f_\Theta(x_1\circ z)-f_{\Theta_{\lnot j}^*}(x_1\circ z)) - g(\theta_j^{*\top}(x_1\circ z))|dx_1
&=\int_{-\epsilon}^\epsilon|\tilde\theta_j^\top (x_1\circ z)-g(\theta_j^{*\top}(x_1\circ z))|dx_1 \\
&=\int_{-\epsilon}^\epsilon|(t_1 x_1 + t_2 z_1) - g(x_1)|dx_1 \\
&\ge \frac{\epsilon^2}{8},
\end{align*}
where the last inequality uses Lemma~\ref{lem:linearbound}. Given the above result, we can derive from \eqref{eq:ineqint_xjp} that
\begin{align*}
\int_{X_j'}|f_\Theta(x)-f_{\Theta^*}(x)|dx
&\ge|Z\setminus Z^{\Theta_{\lnot j}^*}|\frac{\epsilon^2}{8}
=\left((1-\epsilon^2)^{\frac{d-2}{2}}|S^{d-2}|-\frac{4k\epsilon|S^{d-3}|}{\alpha}\right)\frac{\epsilon^2}{8}.
\end{align*}
This combined with Lemma~\ref{lem:xapprox} gives that 
\begin{align*}
\int_{X_j}|f_\Theta(x)-f_{\Theta^*}(x)|dx
&\ge\left((1-\frac{d}{2}\epsilon^2)|S^{d-2}|-\frac{4k\epsilon|S^{d-3}|}{\alpha}\right)\frac{\epsilon^2}{8} - 6kd\epsilon^3|S^{d-2}|.
\end{align*}
The result then follows.
\end{proof}

\subsection{Proof of Lemma~\ref{lem:xapprox}}
\label{sec:lem:xapprox:proof}

First, note that
\begin{align*}
\int_{X_j'}|f_\Theta(x)-f_{\Theta^*}(x)|dx
&=\int_{X_j}|f_\Theta(\phi(x))-f_{\Theta^*}(\phi(x))|\cdot|\det\nabla_x\phi(x)|dx \\
&=\int_{X_j}|f_\Theta(\phi(x))-f_{\Theta^*}(\phi(x))|\left(\frac{1-\epsilon^2}{1-x_1^2}\right)^{\frac{d-1}{2}}dx,
\end{align*}
where the second equality follows the fact that $\nabla_x\phi(x)$ is a lower triangular matrix. Then, we have
\begin{align*}
& \int_{X_j'}|f_\Theta(x)-f_{\Theta^*}(x)|dx - \int_{X_j}|f_\Theta(x)-f_{\Theta^*}(x)|dx\\
=&\int_{X_j}\left(\frac{1-\epsilon^2}{1-x_1^2}\right)^{\frac{d-1}{2}}|f_\Theta(\phi(x))-f_{\Theta^*}(\phi(x))|-|f_\Theta(x)-f_{\Theta^*}(x)|dx \\
\le&\int_{X_j}|(f_\Theta(x)-f_\Theta(\phi(x)))-(f_{\Theta^*}(x)-f_{\Theta^*}(\phi(x)))|dx \\
\le&\int_{X_j}|f_\Theta(x)-f_\Theta(\phi(x))|+|f_{\Theta^*}(x)-f_{\Theta^*}(\phi(x))|dx,
\end{align*}
and
\begin{align*}
& \int_{X_j}|f_\Theta(x)-f_{\Theta^*}(x)|dx-\int_{X_j'}|f_\Theta(x)-f_{\Theta^*}(x)|dx\\
=&\int_{X_j}|f_\Theta(x)-f_{\Theta^*}(x)|-\left(\frac{1-\epsilon^2}{1-x_1^2}\right)^{\frac{d-1}{2}}|f_\Theta(\phi(x))-f_{\Theta^*}(\phi(x))|dx \\
\le&\int_{X_j}|f_\Theta(x)-f_{\Theta^*}(x)|-(1-\epsilon^2)^{\frac{d-1}{2}}|f_\Theta(\phi(x))-f_{\Theta^*}(\phi(x))|dx \\
\le&\int_{X_j}|f_\Theta(x)-f_{\Theta^*}(x)|-(1-\frac{d}{2}\epsilon^2)|f_\Theta(\phi(x))-f_{\Theta^*}(\phi(x))|dx \\
\le&\int_{X_j}|f_\Theta(x)-f_\Theta(\phi(x))|+|f_{\Theta^*}(x)-f_{\Theta^*}(\phi(x))|dx \\
& \qquad + \frac{d}{2}\epsilon^2\int_{X_j}|f_\Theta(\phi(x))-f_{\Theta^*}(\phi(x))|dx,
\end{align*}
where we use $|x_1|\le\epsilon$ and Bernoulli's inequality. Therefore, we have
\begin{multline*}
\left|\int_{X_j}|f_\Theta(x)-f_{\Theta^*}(x)|dx-\int_{X_j'}|f_\Theta(x)-f_{\Theta^*}(x)|dx\right|\\
\le\int_{X_j}|f_\Theta(x)-f_\Theta(\phi(x))|+|f_{\Theta^*}(x)-f_{\Theta^*}(\phi(x))|dx + \frac{d}{2}\epsilon^2\int_{X_j}|f_\Theta(\phi(x))-f_{\Theta^*}(\phi(x))|dx.
\end{multline*}
Note that both $f_\Theta$ and $f_{\Theta^*}$ are $k$-Lipschitz. In particular, for any $x,x'\in\mathbb{R}^d$, we have
\begin{align*}
|f_\Theta(x)-f_\Theta(x')|
=\left|\sum_{i=1}^k (g(\theta_i^\top x)-g(\theta_i^\top x'))\right|
\le\sum_{i=1}^k|\theta_i^\top(x-x')|
&\le k\|x-x'\|_2.
\end{align*}
The same result holds for $f_{\Theta^*}$. As a result, we can derive 
\begin{align*}
\int_{X_j}|f_\Theta(x)-f_\Theta(\phi(x))|+|f_{\Theta^*}(x)-f_{\Theta^*}(\phi(x))|dx
\le 2k |X_j| \max_{x\in X_j}\|x-\phi(x)\|,
\end{align*}
and 
\begin{align*}
\int_{X_j}|f_\Theta(\phi(x))-f_{\Theta^*}(\phi(x))|dx
\le 2k |X_j| \max_{x\in X_j}\|\phi(x)\|.
\end{align*}
Note that
\begin{align*}
\max_{x\in X_j} \|x-\phi(x)\|_2
=\max_{x\in X_j} \sqrt{1-x_1^2}-\sqrt{1-\epsilon^2}
\le\epsilon^2,
\end{align*}
and 
\begin{align*}
\max_{x\in X_j} \|\phi(x)\|_2 \le 1.
\end{align*}
Besides, we have
\begin{align*}
|X_j|
=\int_{-\epsilon}^{\epsilon}A^{d-2}\left(\sqrt{1-x_1^2}\right)dx_1
\le2\epsilon|S^{d-2}|,
\end{align*}
where $A^n(r)$ is the area of the $n$-sphere with radius $r$. Then, the claim follows.

\subsection{Proof of Lemma~\ref{lem:linearregion}}
\label{sec:lem:relulinear:proof}

Consider the set $Z^{\theta_i}$ for some $i\in[k]$ first. Without loss of generality, we can modify the coordinate system so that
\begin{align*}
\theta_i=\begin{bmatrix}t_1&t_2&0& \dots &0\end{bmatrix}^\top
\end{align*}
without affecting $\theta_j^*$ in (\ref{eq:wlog_thetajs}). By assumption, we have $\|\theta_i\|_2=\sqrt{t_1^2+t_2^2}=1$. In the following, we first consider $t_1\ge0$. Remember
\begin{align*}
\alpha^2<\|\theta_i-\theta_j^*\|_2^2
&=(1-t_1)^2+t_2^2 =2(1-t_1)
=2\left(1-\frac{t_1}{\sqrt{t_1^2+t_2^2}}\right)
=2\left(1-\frac{1}{\sqrt{1+t_2^2/t_1^2}}\right).
\end{align*}
This implies
\begin{align*}
\frac{|t_2|}{|t_1|}
>\sqrt{\left(\frac{1}{1-\alpha^2/2}\right)^2-1}
\ge\sqrt{\left(1+\frac{\alpha^2}{2}\right)^2-1}
\ge\alpha
\end{align*}
For any $z\in Z^{\theta_i}$, the condition $\theta_i^\top(x_1\circ z)=0$ is equivalent to
\begin{align*}
t_1x_1+t_2z_1=0.
\end{align*}
As a consequence, we have
\begin{align*}
|z_1|
\le\frac{|t_1|\cdot|x_1|}{|t_2|}
\le\frac{\epsilon}{\alpha}.
\end{align*}
Therefore, we can obtain that
\begin{align*}
|Z^{\theta_j}|
\le\int_{-\epsilon/\alpha}^{\epsilon/\alpha}A^{d-3}\left(\sqrt{1-\epsilon^2-z_1^2}\right)dz_1
\le\frac{2\epsilon}{\alpha}|S^{d-3}|,
\end{align*}
where $A^n(r)$ is the area of the $n$-sphere with radius $r$. The claim for $\theta\in\{\theta_i\}_{i\in[k]}$ then follows. The case $t_1 < 0$ can be analyzed similarly using $\alpha < \|\theta_i + \theta_j^*\|_2$. In addition, the claim for $\theta\in\{\theta^*_i\}_{i\in[k], i\ne j}$ also holds noticing $2\alpha < \alpha_0 < \|\theta_i^* \pm \theta_j^*\|_2$ for $i\ne j$ by Assumption~\ref{assump:neuronskip}.

\subsection{Proof of Lemma~\ref{lem:linearbound}}
\label{sec:lem:linearbound:proof}

We first have
\begin{align*}
\int_{-\epsilon}^\epsilon|(\beta_0+\beta_1w)-g(w)|dw
&=\int_{-\epsilon}^0|\beta_0+\beta_1w|dw+\int_0^\epsilon|\beta_0+(\beta_1-1)w|dw \\
&=\int_0^\epsilon|\beta_0-\beta_1w|dw+\int_0^\epsilon|\beta_0+(\beta_1-1)w|dw \\
&=F(\beta_0,\beta_1)+F(\beta_0,1-\beta_1),
\end{align*}
where we define $F(a,b)=\int_0^\epsilon|a-bw|dw$. Note that we must have either $\beta_1\ge1/2$ or $1-\beta_1\ge1/2$. Additionally, it holds that $F(-|a|, b) \ge F(|a|, b)$ for $b>0$. Therefore, it suffices to consider the case $a\ge0,b>1/2$ to provide a lower bound. In this case, $F(a, b)$ takes the minimum when $a\le b\epsilon$. Therefore, we have
\begin{align*}
F(a, b)& = \int_0^{a/b}(a-bw)dw+\int_{a/b}^\epsilon(bw-a)dw \\
&=\left(\frac{a^2}{b}-\frac{a^2}{2b}\right)+\left(\frac{b\epsilon^2}{2}-a\epsilon\right)-\left(\frac{a^2}{2b}-\frac{a^2}{b}\right) \\
&=\frac{a^2}{b}+\frac{b\epsilon^2}{2}-a\epsilon.
\end{align*}
As a function of $a$, this expression is minimized when $a=b\epsilon/2$. In such case,
\begin{align*}
F(a, b) \ge F\left(\frac{b\epsilon}{2},b\right) =\frac{b\epsilon^2}{4}\ge\frac{\epsilon^2}{8},
\end{align*}
where we use $b \ge 1/2$. The result then follows.

\section{Proof of Theorem~\ref{thm:key}}
\label{sec:thm:key:proof}

We give the proof of Theorem~\ref{thm:key}, followed by proofs of the lemmas used in this proof.
First, we have the following technical lemma:
\begin{lemma}
\label{lem:relulipschitz}
We have (a) the mean squared loss $L_{S,p}$ is $4k$-Lipschitz in $\Theta$ with respect to $\ell_{2,1}$ norm, and (b) when $n\ge\log(\frac{2}{\delta})$, its empirical counterpart $\hat{L}_{S,p}$ is $(4k+6\sigma)$-Lipschitz with at least a probability of $1-\delta/2$.
\end{lemma}
We give a proof in Appendix~\ref{sec:lem:relulipschitz:proof}.
Next, we have the following useful lemma, which says that the empirical loss converges to the true loss.
\begin{lemma}
\label{lem:lossconvergence}
Given the same setup as in Lemma~\ref{lem:empiricalloss}, we have
\begin{align*}
\mathbb{P}_{p}\left[\sup_{\Theta}|(\hat{L}_S(\Theta;Z)-\sigma(\xi))-L_{S,p}(\Theta)|\le\zeta\right]\ge1-\delta,
\end{align*}
where $\sigma(\xi)\coloneqq\frac{1}{n}\sum_{i\in[n]}\xi_i^2$.
\end{lemma}
The proof is given in Appendix~\ref{sec:lem:lossconvergence:proof}. This lemma follows from a standard covering number argument. Using this result, we provide a sample complexity bound for learning the true parameter given the i.i.d. training examples $\{(x_i, y_i)\}_{i\in[n]}$ from $p$.  

Now we define the mean absolute loss function and its corresponding empirical version:
\begin{align*}
L_{A,p}(\Theta)=\mathbb{E}_p\left[|f_\Theta(x)-f_{\Theta^*}(x)|\right], \quad
\hat{L}_A(\Theta;Z)=\frac{1}{n}\sum_{i\in[n]}|f_{\Theta}(x_i)-y_i|.
\end{align*}
\begin{lemma}
\label{lem:empiricalloss} 
For any $\delta\in\mathbb{R}_{>0}$, we have
\begin{align*}
\mathbb{P}_{p}\left[L_{A,p}(\hat\Theta)\le\sqrt{2\zeta}\right]\ge1-\delta,
\end{align*}
where $\zeta$ is defined as
\begin{align*}
\zeta =
\sqrt{\frac{4096k^2(k\lor\sigma)^2}{n}\left(dk \max\left\{1, \log\left(1 + \sqrt{\frac{n}{dk}}\right)\right\} + \log\frac{4}{\delta}\right)}. 
\end{align*}
\end{lemma}
We give a proof in Appendix~\ref{sec:lem:empiricalloss:proof}.

\begin{lemma}
\label{lem:relumain}
Consider any small $\tilde\eta\in\mathbb{R}_{>0}$ satisfying
$\tilde\eta\le\min\{\frac{1}{1152^2\sqrt{2\pi}k^2d^{3/2}},\frac{\alpha_0^2\pi^{1/2}}{1454^2k^2d^{1/2}}\}$. If the generalization error has $L_{A,p}(\Theta)\le\tilde\eta$,
where $p$ is a distribution on the $(d-1)$-sphere such that
\begin{align*}
\frac{1}{|S^{d-1}|}\int_{\mathcal{X}}|f_\Theta(x)-f_{\Theta^*}(x)|dx
\le\mathbb{E}_{p}\left[|f_\Theta(x)-f_{\Theta^*}(x)|\right],
\end{align*}
then there exists a bijection $\sigma:[k]\to[k]$ such that 
\begin{align*}
\min\{\|\theta_{\sigma(i)}-\theta_i^*\|_2,\|\theta_{\sigma(i)}+\theta_i^*\|_2\}
\le\tilde\alpha\coloneqq 727\pi^{-\frac{1}{4}}kd^{\frac{1}{4}}\tilde{\eta}^{\frac{1}{2}}, \quad \forall i \in[k].
\end{align*}
\end{lemma}
We give a proof in Appendix~\ref{sec:lem:relumain:proof}. Note that the assumptions on the distribution $p$ can be easily satisfied, e.g., when $p=\text{Uniform}(S^{d-1})$ is a uniform distribution on the sphere. 
The above result says given a small generalization error for certain parameter estimate $\Theta$, its estimation error of the ground-truth parameter $\Theta^*$ up to a sign flip is also small correspondingly.

Finally, Theorem~\ref{thm:key} follows directly from Lemmas~\ref{lem:empiricalloss} \&~\ref{lem:relumain} by taking $\tilde\eta=\sqrt{2\zeta}$. 

\subsection{Proof of Lemma~\ref{lem:relulipschitz}}
\label{sec:lem:relulipschitz:proof}

The mean squared loss satisfies
\begin{align*}
|L_{S,p}(\Theta) - L_{S,p}(\Theta')|
& \le \mathbb{E}_{p}[|(f_{\Theta}(x)-f_{\Theta^*}(x))^2 - (f_{\Theta'}(x)-f_{\Theta^*}(x))^2|] \\
& \le \mathbb{E}_{p}[|(f_{\Theta}(x)-f_{\Theta^*}(x)+f_{\Theta'}(x)-f_{\Theta^*}(x))(f_{\Theta}(x)-f_{\Theta'}(x))|] \\
& \le 4k \mathbb{E}_{p}[|f_{\Theta}(x)-f_{\Theta'}(x)|] \\
& \le 4k\sum_{i\in[k]} \|\theta_i - \theta_i'\|_2,
\end{align*}
where we use $\|x\|_2=1$ and $\|\theta_i\|_2=\|\theta_i'\|_2=1$.
Similarly, the empirical loss satisfies
\begin{multline*}
|\hat{L}_S(\Theta;Z) - \hat{L}_S(\Theta';Z)|
=\left|\frac{1}{n}\sum_{i\in[n]} [(f_{\Theta}(x_i)-y_i)^2 - (f_{\Theta'}(x_i)-y_i)^2]\right| \\
\le \frac{1}{n}\sum_{i\in[n]} |(f_{\Theta}(x_i)-f_{\Theta^*}(x_i))^2 - (f_{\Theta'}(x_i)-f_{\Theta^*}(x_i))^2| + \frac{2}{n}\sum_{i\in[n]}|\xi_i(f_{\Theta}(x_i) - f_{\Theta'}(x_i))| \\
\le (4k + \frac{2}{n}\sum_{i\in[n]}|\xi_i|)\sum_{i\in[k]} \|\theta_i - \theta_i'\|_2.
\end{multline*}
Since $\xi_i$ is $\sigma$-subgaussian, so is $|\xi_i|$.
When $n \ge \log(\frac{2}{\delta})$, we have 
\begin{align*}
\frac{2}{n}\sum_{i\in[n]}|\xi_i| \le 2\left(E[|\xi_i|] + \sqrt{\frac{2\sigma^2\log(\frac{2}{\delta})}{n}}\right) \le 6\sigma
\end{align*}
with at least a probability of $1 - \delta/2$, where we use Lemma~1.4 from \cite{rigollet2015high} and Lemma~\ref{lem:subg-hoeffding} by taking $t=\sqrt{2\sigma^2\log(\frac{2}{\delta})/n}$. The claim follows. 

\subsection{Proof of Lemma~\ref{lem:lossconvergence}}
\label{sec:lem:lossconvergence:proof}

Then, construct an $\zeta/(16k+12\sigma)$-net $\mathcal{E}_{\Theta}$ with respect to $\ell_{2,1}$ norm of $\bar\varTheta=\{\Theta\in\mathbb{R}^{k\times d} \mid \|\Theta\|_{2,1}=k\}$ (remember $\bar\varTheta$ is a superset of $\Theta$'s domain $\varTheta$). For any such $\Theta\in\bar\varTheta$, there exists $\Theta'\in\mathcal{E}_{\Theta}$ such that
\begin{align*}
|(\hat{L}_S(\Theta;Z)-L_{S,p}(\Theta)) - (\hat{L}_S(\Theta'; Z)-L_{S,p}(\Theta'))| \le (8k+6\sigma)\|\Theta - \Theta'\|_{2,1} \le \zeta/2,
\end{align*}
with high probability $1-\delta/2$.
Given the above inequality, we have
\begin{multline}
\label{eqn:lem:relulossbound:1}
\mathbb{P}_p\left[\sup_{\Theta\in\varTheta} |(\hat{L}_S(\Theta;Z)-\sigma(\xi))-L_{S,p}(\Theta)|\ge \zeta\right]
\le \mathbb{P}_p\left[\sup_{\Theta\in\bar\varTheta} |(\hat{L}_S(\Theta;Z)-\sigma(\xi))-L_{S,p}(\Theta)|\ge \zeta\right] \\
\le \mathbb{P}_p\left[\max_{\Theta \in \mathcal{E}_{\Theta}} |(\hat{L}_S(\Theta;Z)-\sigma(\xi))-L_{S,p}(\Theta)| \ge \frac{\zeta}{2}\right] + \delta/2 \\
\le \sum_{\Theta\in\mathcal{E}_\Theta} \mathbb{P}_{p}\left[|(\hat{L}_S(\Theta;Z)-\sigma(\xi))-L_{S,p}(\Theta)| \ge \frac{\zeta}{2}\right] + \delta/2.
\end{multline}
Note that
\begin{align*}
(\hat{L}_S(\Theta;Z)-\sigma(\xi))-L_{S,p}(\Theta) = \mathcal{L}_1 + 2\mathcal{L}_2, 
\end{align*}
where 
\begin{align*}
\mathcal{L}_1 = \frac{1}{n}\sum_{i\in[n]}(f_\Theta(x_i)-f_{\Theta^*}(x_i))^2 - \mathbb{E}_{p}\left[(f_\Theta(x)-f_{\Theta^*}(x))^2\right], \quad 
\mathcal{L}_2 = \frac{1}{n}\sum_{i\in[n]}(f_\Theta(x_i)-f_{\Theta^*}(x_i))\xi_i.
\end{align*}
The probability within the sum in \eqref{eqn:lem:relulossbound:1} is then upper bounded by 
\begin{align*}
\mathbb{P}_{p}\left[|(\hat{L}_S(\Theta;Z)-\sigma(\xi))-L_{S,p}(\Theta)| \ge \frac{\zeta}{2}\right] 
\le \mathbb{P}_{p}\left[|\mathcal{L}_1| \ge \frac{\zeta}{4}\right] + \mathbb{P}_{p}\left[|\mathcal{L}_2| \ge \frac{\zeta}{8}\right].
\end{align*}
In the following, we bound each of the above two terms respectively. For the first term, 
since $|f_\Theta(x_i)-f_{\Theta^*}(x_i)|\le2k$, $(f_\Theta(x_i)-f_{\Theta^*}(x_i))^2$ is $4k^2$-subgaussian; thus, by Lemma \ref{lem:subg-hoeffding}, we have
\begin{align}
\mathbb{P}_{p}\left[|\mathcal{L}_1| \ge \frac{\zeta}{4}\right]
& \le 2\exp\left(-\frac{n\zeta^2}{2048k^4}\right). \label{eqn:lem:lossbound:3}
\end{align}
Next, for the second term, since $\xi_i$ is $\sigma$-subgaussian and $(f_\Theta(x_i)-f_{\Theta^*}(x_i))$ and $\xi_i$ are independent, we can show that $(f_\Theta(x_i)-f_{\Theta^*}(x_i))\xi_i$ is $(2k\sigma)$-subgaussian; thus, by Lemma~\ref{lem:subg-hoeffding}, we have
\begin{align}
\mathbb{P}_p\left[|\mathcal{L}_2| \ge \frac{\zeta}{8} \right]
&\leq  2\exp\left(-\frac{n\zeta^2}{512k^2\sigma^2} \right). \label{eqn:lem:lossbound:4}
\end{align}
Therefore, we obtain from (\ref{eqn:lem:relulossbound:1}) that
\begin{multline}
\label{eqn:lem:lossbound:5}
\mathbb{P}_p\left[\sup_{\Theta\in\varTheta} |(\hat{L}_S(\Theta;Z)-\sigma(\xi))-L_{S,p}(\Theta)|\ge \zeta\right] 
\le4|\mathcal{E}_\Theta|\exp\left(-\frac{n\zeta^2}{2048k^2(k^2+\sigma^2)}\right) \\
\le 2\left(1+\frac{2k(16k+12\sigma)}{\zeta}\right)^{dk}\exp\left(-\frac{n\zeta^2}{2048k^2(k^2+\sigma^2)}\right) \\
\le 2\exp\left(-\frac{n\zeta^2}{2048k^2(k^2+\sigma^2)}+dk\log\left(1+\frac{2k(16k+12\sigma)}{\zeta}\right)\right), 
\end{multline}
where the second inequality follows Lemma~\ref{lem:covnum} by noticing that $\|\Theta\|_{2,1}=k$ for $\Theta\in\bar\varTheta$.
Finally, we choose $\zeta$ so that (\ref{eqn:lem:lossbound:5}) is smaller than $\delta/2$---in particular, letting
\begin{align*}
\zeta = \sqrt{\frac{4096k^2(k\lor\sigma)^2}{n}\left(dk \max\left\{1, \log\left(1 + \sqrt{\frac{n}{dk}}\right)\right\} + \log\frac{4}{\delta}\right)},
\end{align*}
then we have \eqref{eqn:lem:lossbound:5} is upper bounded by $\delta/2$, where we use $k+\sigma\ge1$. The result follows. 

\subsection{Proof of Lemma~\ref{lem:empiricalloss}}
\label{sec:lem:empiricalloss:proof}

Since $\hat\Theta$ minimizes $\hat{L}_S(\Theta;Z)$, we have
\begin{align*}
0 & \le L_{S,p}(\hat\Theta) - L_{S,p}(\Theta^*) \\
& \le L_{S,p}(\hat\Theta) - (\hat{L}_S(\hat\Theta;Z)-\sigma(\xi)) + (\hat{L}_S(\Theta^*;Z)-\sigma(\xi)) - L_{S,p}(\Theta^*) \\
& \le 2 \sup_{\Theta}|(\hat{L}_S(\Theta;Z)-\sigma(\xi))-L_{S,p}(\Theta)| \le 2\zeta,
\end{align*}
with probability at least $1-\delta$, where we use Lemma~\ref{lem:lossconvergence}. Thus, by Cauchy-Schwarz inequaltiy and the fact that $L_{S,p}(\Theta^*)=0$, we have
\begin{align*}
L_{A,p}(\hat\Theta)
\le \sqrt{L_{S,p}(\hat\Theta)}
\le\sqrt{2\zeta},
\end{align*}
as claimed. 

\subsection{Proof of Lemma~\ref{lem:relumain}}
\label{sec:lem:relumain:proof}

Using the condition on $p$, we have
\begin{align*}
\int_{X_i}|f_\Theta(x)-f_{\Theta^*}(x)|dx
&\le\int_{\mathcal{X}}|f_\Theta(x)-f_{\Theta^*}(x)|dx
\le|S^{d-1}| L_{A,p}(\Theta)
\le|S^{d-1}|\tilde{\eta}.
\end{align*}
Then, letting $\eta=|S^{d-1}|\tilde\eta$, we have that $L_{X_i}(\Theta) \le \eta$ holds for all $i\in[k]$. Thus, by Proposition~\ref{prop:key} (we will check the conditions on $\alpha$ later), we have
\begin{align*}
\min\{\|\theta_{\sigma(i)}-\theta_i^*\|_2,
\|\theta_{\sigma(i)}+\theta_i^*\|_2\}\le \alpha
\end{align*}
for all $i\in[k]$. Plugging the value of $\eta$ above into $\alpha$, it yields that
\begin{align}
\label{eq:alphaineq}
\alpha
&=\frac{\frac{k\epsilon^3}{2}\frac{|S^{d-3}|}{|S^{d-1}|}}{\frac{\epsilon^2(1-d\epsilon^2/2)}{8}\frac{|S^{d-2}|}{|S^{d-1}|}-\tilde\eta-6kd\epsilon^3\frac{|S^{d-2}|}{|S^{d-1}|}}.
\end{align}
Using $|S^{d-1}|=\frac{2\pi^{d/2}}{\Gamma(d/2)}$ and Lemma~10 from \cite{xu2021group}, it holds that for any $d>1$
\begin{align*}
\frac{|S^{d-3}|}{|S^{d-1}|}\le\frac{d}{2\pi}, \qquad 
\frac{|S^{d-2}|}{|S^{d-1}|}\ge\frac{\sqrt{d}}{2\sqrt{2\pi}}.
\end{align*}
Combining the above, we obtain from \eqref{eq:alphaineq} 
\begin{align*}
\alpha&\le\frac{4\sqrt{2}kd\epsilon^3}{(\pi d)^{1/2}(1-d\epsilon^2/2-48kd\epsilon)\epsilon^2-16\sqrt{2}\pi\tilde\eta}.
\end{align*}
Take $\epsilon=(36\sqrt{2\pi}\frac{\tilde{\eta}}{\sqrt{d}})^{1/2}$. Given our choice of $\tilde\eta$, we can show that $d\epsilon^2/2\le 1/4$ and $48kd\epsilon\le 1/4$. Thus, we have
\begin{align*}
\alpha\le 727\pi^{-\frac{1}{4}}kd^{\frac{1}{4}}\tilde{\eta}^{\frac{1}{2}}.
\end{align*}
Note that such $\alpha$ satisfies $\alpha\le\alpha_0/2$ by our condition on $\tilde\eta$. The claim follows. 

\section{Proof of Proposition~\ref{prop:optimalactiongap}}
\label{sec:prop:optimalactiongap:proof}

We give a proof of Proposition~\ref{prop:optimalactiongap}.

\subsection{Intuition}

We illustrate our proof strategy in Figure~\ref{fig:proofsketch_relub}.
For all the neurons $\theta_i^*$'s, we can decompose them into three subsets --- i.e., the neurons that are ``positively'' activated, $A=\{i\in[k]\mid \theta_i^{*\top}x^*>0\}$, those that are orthogonal to $x^*$, $B=\{i\in[k]\mid \theta_i^{*\top}x^*=0\}$, and those that are inactive, $C=\{i\in[k]\mid \theta_i^{*\top}x^*<0\}$. Define $\bar\theta_A^* = \sum_{i\in A}\theta_i^*$. We first show that $\bar\theta_A^*$ and $x^*$ should be in the same direction (Figure \ref{fig:proofsketch_relub} (a)). Then, given this fact, we prove by showing that there is no neuron in the set $B$; otherwise, we can always find an action $x'$ to improve the value of $f_{\Theta^*}(x)$ (Figure \ref{fig:proofsketch_relub} (b)). 

\begin{figure*}[t]
\centering
\begin{subfigure}[b]{0.3\textwidth}
  \centering
  \includegraphics[width=\textwidth]{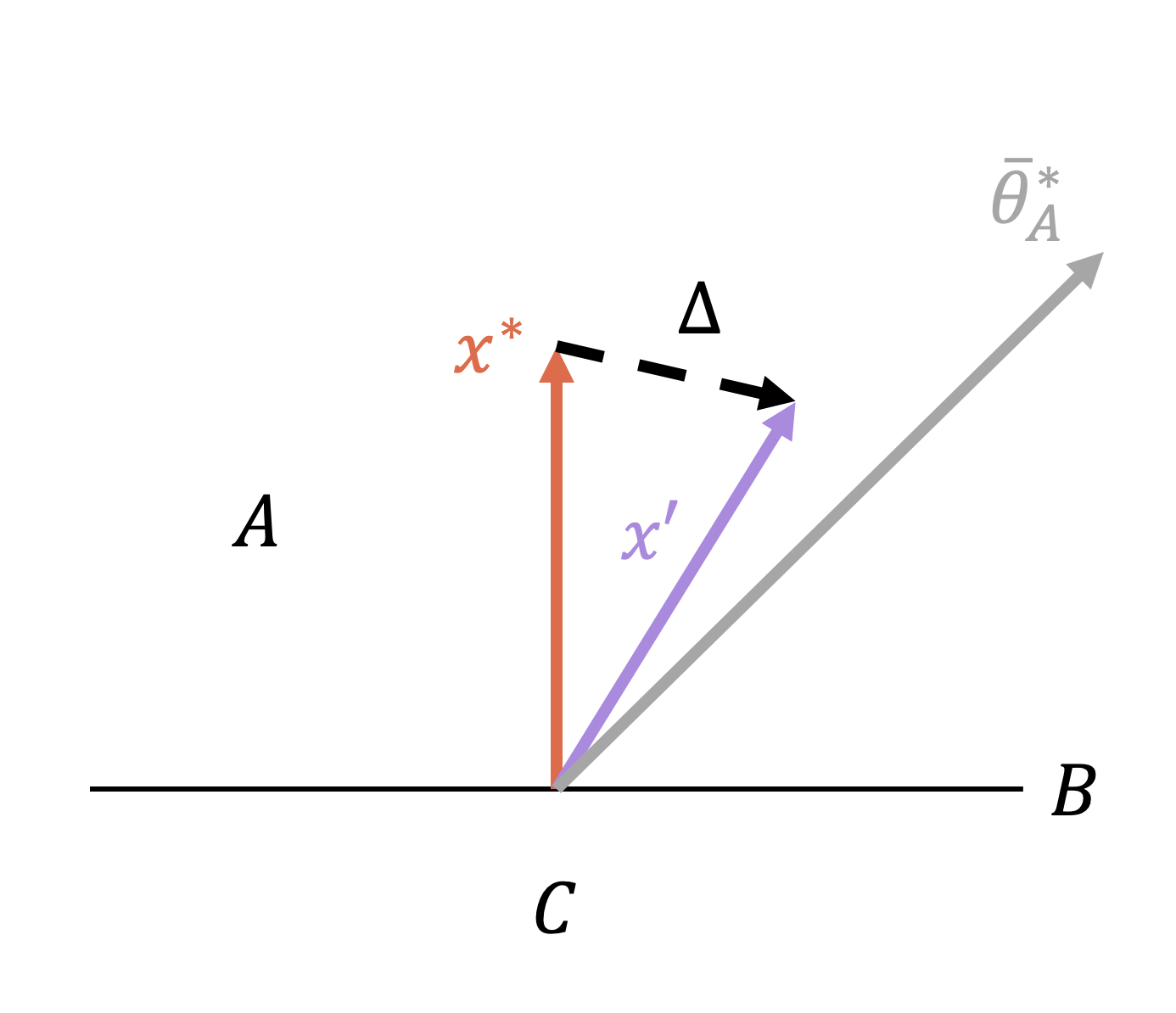}
  \caption{}
\end{subfigure}
\begin{subfigure}[b]{0.3\textwidth}
  \centering
  \includegraphics[width=\textwidth]
  {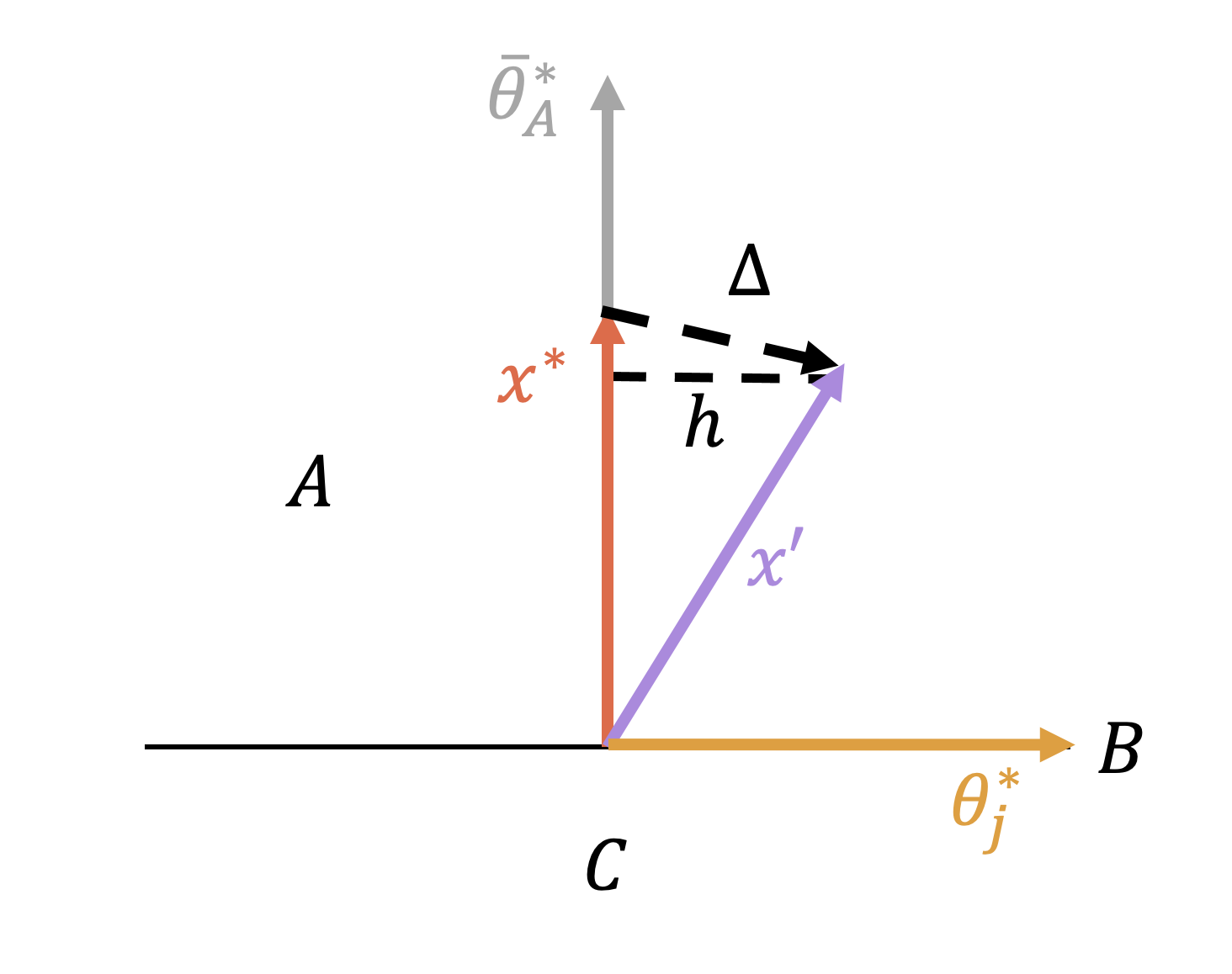}
  \caption{}
\end{subfigure}
\caption{Illustrations for proof sketch of the optimal action gap $\nu_*$.}
\label{fig:proofsketch_relub}
\end{figure*}

To see that, the maximum value of $f_{\Theta^*}(x)$ is equal to $f_{\Theta^*}(x^*) = \bar\theta_A^{*\top}x^*$. Therefore, if $\bar\theta_A^*$ is not aligned with $x^*$, we can always move $x^*$ a bit closer to $\bar\theta_A^*$ to a position $x'$ such that $\bar\theta_A^{*\top}x' > \bar\theta_A^{*\top}x^*$ without negatively affecting the other neurons in $B\cup C$. Then, suppose the set $B\ne\varnothing$ and there's at least one $j\in B$. We can always find an increasing direction $\Delta$ by moving $x^*$ a bit closer to $\theta_j^*$. Intuitively, $\Delta$ is almost orthogonal to $\bar\theta_A^*$ and in parallel with $\theta_j^*$; thus, the increase by moving $x^*$ closer to $\theta_j^*$ exceeds the the decrease by moving it away from $\bar\theta_A^*$ when $\Delta$ is small enough. 

\subsection{Proof of Proposition~\ref{prop:optimalactiongap}}

First, we show that $\bar\theta_A^*$ and $x^*$ are in the same direction. We prove by contradiction. Suppose $\bar\theta_A^*$ and $x^*$ are not aligned. We will show that we can find an action $x'$ such that $f_{\Theta^*}(x') > f_{\Theta^*}(x^*)$. Take $x'$ to be $\Delta$ away from $x^*$, i.e., $x' = x^* + \Delta$, such that $x'$, $x^*$ and $\bar\theta_A^*$ are in the same hyperplane (see Figure \ref{fig:proofsketch_relub} (a)). We first let $\Delta$ be small enough such that the neurons in $A$ remains strictly activated, i.e., $\theta_i^{*\top}x'>0$ for $i\in A$ (it suffices to take $\|\Delta\|_2<\min_{i\in A}|\theta_i^{*\top}x^*|$).
Then, define the angle between two vectors $z$ and $y$ as $\angle(z,y)\coloneqq \arccos(\frac{z^\top y}{\|z\|_2\|y\|_2})$. 

As long as 
$\Delta$ is such that $\angle(x^*, \bar\theta_A^*) > \angle(x', \bar\theta_A^*)$, it's easy to show that $\bar\theta_A^{*\top}x^* < \bar\theta_A^{*\top}x'$ by noting that $\|x^*\|_2=\|x'\|_2=1$ and $\angle(x^*, \bar\theta_A^*),\angle(x', \bar\theta_A^*) < \pi/2$. Additionally, we also require $\Delta$ to be small enough such that the neurons in $C$ remains inactive choosing action $x'$, i.e., $\theta_i^{*\top}x'<0$ for $i\in C$; note that it suffices to take $\|\Delta\|_2<\min_{i\in C}|\theta_i^{*\top}x^*|$. Finally, we have $\sum_{i\in B}g(\theta_i^{*\top}x')\ge\sum_{i\in B}g(\theta_i^{*\top}x^*)=0$ since $g(z)$ is non-negative. Therefore, we conclude that 
\begin{align*}
f_{\Theta^*}(x') = \bar\theta_A^{*\top}x' + \sum_{i\in B}g(\theta_i^{*\top}x') > \bar\theta_A^{*\top}x^* + \sum_{i\in B}g(\theta_i^{*\top}x^*) = f_{\Theta^*}(x^*),
\end{align*}
which is a contradiction. 

Now, we further show that $B=\varnothing$. Similarly, we use a proof by contradiction. Suppose there is at least some $j\in B$. We take an action $x' = x^* + \Delta$, such that $x'$, $x^*$ and $\theta_j^*$ are in the same hyperplane (see Figure \ref{fig:proofsketch_relub} (b)). Note that both $\bar\theta_A^*$ and $x^*$ are orthogonal to $\theta_j^*$ by definition of the set $B$. Then, it holds that $\bar\theta_A^{*\top}x' + \theta_j^{*\top}x' > \bar\theta_A^{*\top}x^* + \theta_j^{*\top}x^*$, as long as $\Delta$ is such that
\begin{align}\label{eq:bempty_suff}
\theta_j^{*\top}\Delta > -\bar\theta_A^{*\top}\Delta.
\end{align}
Since $\|x^*\|=\|x'\|=1$ and $x' = x^* + \Delta$, we have $-2x^{*\top}\Delta = \|\Delta\|^2$. 
Besides, let $h$ denote the perpendicular distance from $x'$ to $x^*$ and it satisfies 
\begin{align*}
\sqrt{1-h^2}+\sqrt{\|\Delta\|_2^2-h^2}=1 
\Leftrightarrow
h=\|\Delta\|_2\sqrt{1-\frac{\|\Delta\|_2^2}{4}}.
\end{align*}
By noting that $\bar\theta_A^*/\|\bar\theta_A^*\|_2=x^*$, it suffices to have 
\begin{align}\label{eq:bempty_suff_h}
\cos(\angle(\theta_j^*,\Delta))\|\Delta\|_2 \ge \frac{1}{2}\|\bar\theta_A^*\|\|\Delta\|_2^2 
\Leftrightarrow
\frac{h}{\|\Delta\|_2} \ge \frac{1}{2}\|\bar\theta_A^*\|\|\Delta\|_2
\Leftrightarrow
\|\Delta\|_2 \le \frac{4}{1+\|\bar\theta_A^*\|_2^2}
\end{align}
so that \eqref{eq:bempty_suff} holds.
In addition, we take $\Delta$ to be small enough such that the neurons in $A$ remains strictly activated and those in $C$ remain inactive; to that end, it suffices to take $\|\Delta\|_2<\min_{i\in A\cup C}|\theta_i^{*\top}x^*|$. Thus, we reach a contradiction that $f_{\Theta^*}(x') > f_{\Theta^*}(x^*)$. Our claim follows. 

\section{Proof of Theorem \ref{thm:algo1_reg}}
\label{app:thm_bandit}

Suppose the exploration stage ends at time $t_0$. By Theorem~\ref{thm:key}, we can ensure
\begin{align*}
\min\{\|\tilde\theta_i-\theta_i^*\|_2, \|\tilde\theta_i+\theta_i^*\|_2\}\le\nu_*/2,
\quad \forall i\in[k]
\end{align*}
with probability at least $1-\delta/2$ by choosing $t_0$ large enough so that $727\pi^{-\frac{1}{4}}kd^{\frac{1}{4}}(2\zeta)^{\frac{1}{4}}\le\nu_*/2$ (note that $\zeta$ depends on the sample size $n=t_0$). In particular, it suffices to take $\delta = 1/\sqrt{T}$ and 
\begin{align}
\label{eq:t0max_1}
t_0\ge t_1(\nu_*),\quad \text{where} \quad t_1(\nu)\coloneqq\frac{C_1 k^{10} d^2(k\lor\sigma)^2}{\nu^8}\left(dk(\log(d(k\lor\sigma))\lor\log\log T)+\log(64T)\right)
\end{align}
for some constant $C_1$. 
Recall that Theorem~\ref{thm:key} requires 
$\sqrt{2\zeta}\le\min\left\{\frac{1}{1152^2\sqrt{2\pi}k^2d^{3/2}},\frac{\alpha_0^2\pi^{1/2}}{1454^2k^2d^{1/2}}\right\}$; thus, we also require
\begin{align}
\label{eq:t0max_2}
t_0\ge t_2\coloneqq C_2k^{10}d^6(k\lor\sigma)^2\left(dk(\log(d(k\lor\sigma))\lor\log\log T)+\log(64T)\right)
\end{align}
for some constant $C_2$. Therefore, we have, with a slight abuse of notation, 
\begin{align}
\label{eq:t0max}
t_0=t_0(\nu_*)=\mathit{\tilde\Theta}(k^{13}d^3(1/\nu_*^8\lor d^4)), \quad \text{where} \quad t_0(\nu) = \max\{t_1(\nu), t_2\}.
\end{align}

Now we analyze the regret of our algorithm. 
At each time $t$, the per-period regret $r_t$ can be upper bounded by 
\begin{align*}
r_t = f_{\Theta^*}(x^*) - f_{\Theta^*}(x_t)
\le (\sum_{i\in[k]} \|\theta_i^*\|_2)(\|x^*\|_2+\|x_t\|_2) 
\le 2k.
\end{align*}
Therefore, the regret during the exploration stage is upper bounded by 
\begin{align*}
\sum_{t\in[t_0]} r_t \le 2kt_0 = \tilde{O}(k^{14}d^3(1/\nu_*^8\lor d^4)).
\end{align*}

In the second stage, we run OFUL to find the optimal policy for the linear function $f_{\theta^\ddagger}(x^\ddagger)$ given our estimate $\tilde\Theta_{t_0}$, where $x^\ddagger$ and $\theta^\ddagger$ are defined in \eqref{def:xtheta_eq_ucb_x} and \eqref{def:xtheta_eq_ucb_tht} respectively. Following the same proof strategy of Theorem 3 in \cite{abbasi2011improved}, it gives the regret bound in the second stage to be 
\begin{align*}
\sum_{t\in[T]\setminus[t_0]}r_t \le C_3\sqrt{kdT\log(\lambda+T/(2kd))}\left(\lambda^{1/2}\sqrt{k}+\sigma\sqrt{\log(4T)+2kd\log(1+T/(2\lambda kd))}\right)
\end{align*}
with a probability at least $1-\delta/2$, where the contextual dimension $d$ here is $2kd$, and $S=\sqrt{5k}$ by noting that $\|\theta^\ddagger\|\le \sqrt{5k}$. 

Finally, the above analysis shows that with a probability at least $1-\delta$, we both have a small estimation error in the exploration stage and a guanranteed regret upper bound of OFUL in the second stage. Thus, with a small probability $\delta=1/\sqrt{T}$, we would have linear regret scaling as $2kT$; thus, the expected regret in this case is bounded by $2kT\delta =2k\sqrt{T}$. Our claim then follows. 

\section{Proof of Theorem \ref{thm:algo2_reg}}
\label{app:pf_ofureluplus}

We bound the regret for the three cases respectively: (i) all batch $i$ satisfying $\nu_i>\nu_*$, (ii) $t\in(T_{i-1}, T_{i-1}+t_{0,i}]$ for all batch $i$ with $\nu_i\le\nu_*$, and (iii) $t\in(T_{i-1}+t_{0,i}, T_{i}]$ for all batch $i$ with $\nu_i\le\nu_*$.

First, in case (i), we have $i \le \log(\nu_0/\nu_*)/\log(b)$. Recall that the per-period regret $r_t$ can be trivially bounded by $2k$. Thus, the regret in this case is upper bounded by 
\begin{align*}
2k (a^{\log(\nu_0/\nu_*)/\log(b)}-1) T_1 \le 2k(\nu_0/\nu_*)^{\frac{\log(a)}{\log(b)}}T_1.
\end{align*}
Once $\nu_i\le \nu_*$, the gap $\nu_i$ is sufficiently accurate that the optimal action $x^*$ becomes feasible in the search region. Then, we can follow our proof strategy in Appendix \ref{app:thm_bandit}. 

Consider case (ii). Due to our exploration strategy, we randomly collect $t_{0, i} = t_0(\nu_i)-t_0(\nu_{i-1})$ samples in each batch $i$ ($t_0(\nu)$ defined in \eqref{eq:t0max}); thus, the total number of random samples we collect over time horizon $T$ is upper bounded by
\begin{align*}
\sum_{i=\lceil\frac{\log(\nu_0/\nu_*)}{\log(b)}\rceil}^{M-1} t_{0,i} \le t_0(\nu_{M-1}).
\end{align*}
Thus, the regret in case (ii) is upper bounded by
\begin{align*}
2kt_0(\nu_{M-1}) 
=& \tilde{O}\left(k^{14}d^7+k^{14}d^3T^{8\frac{\log(b)}{\log(a)}}\right),
\end{align*}
where we use the definition of $t_0$ in \eqref{eq:t0max_2} and the value of $M$ in \eqref{eq:valM}.

Next, we calculate the regret of running OFUL in case (iii). Same as the proof in Appendix \ref{app:thm_bandit}, we have
\begin{align*}
\sum_{i=\lceil\frac{\log(\nu_0/\nu_*)}{\log(b)}\rceil}^{M-1}\sum_{t=T_{i-1}+t_{0, i}}^{T_{i}}r_t = \tilde{O}(kd\sqrt{T}).
\end{align*}
Note that the proof strategy in \cite{abbasi2011improved} works as long as the data are fetched sequentially, and the length of the time periods without model misspecification (i.e., batch $i$ with $\nu_i<\nu_*$) scales as $T$, which both satisfy in our algorithm.

Finally, recall that at each batch $i$, there's a probability of $\delta_i=1/\sqrt{T}$ that our analysis above will fail. Thus, with a probability of at most $M/\sqrt{T}$ across all $M$ batches, we will have a linear regret. The expected regret for this small-probability event is upper bounded by
\begin{align*}
2kT\cdot M/\sqrt{T} = \tilde{O}(k\sqrt{T}).
\end{align*}
Combining all the above gives our final result.

\section{Useful Lemmas}

\begin{lemma}\label{lem:covnum}
For a ball in $\mathbb{R}^{d_1 \times d_2}$ with radius $r$ with respect to any norm, there exists an $\zeta$-net $\mathcal{E}$ such that
\begin{align*}
|\mathcal{E}| \le \left(1 + \frac{2r}{\zeta}\right)^{d_1d_2}.
\end{align*}
\end{lemma}
\begin{proof}
This claim follows by a direct application of Proposition 4.2.12 in \cite{vershynin2018high}.
\end{proof}
\begin{lemma}
\label{lem:subg-hoeffding}
Letting $\{x_i\}_{i\in[n]}$ be a set of independent $\sigma$-subgaussian random variables with mean $\mu_i$, then for all $t \geq 0$, we have
\begin{align*}
\Pr\left[|\frac{1}{n}\sum_{i\in[n]} (x_i - \mu_i)| \ge t\right] \le 2\exp\left(-\frac{nt^2}{2\sigma^2}\right).
\end{align*}
\end{lemma}
\begin{proof}
See Proposition 2.5 of \cite{wainwright2019high}.
\end{proof}


\end{document}